\documentclass[11pt, a4paper, copyright]{google}

\usepackage[authoryear, sort&compress, round]{natbib}
\makeatletter
\renewcommand{\absfont}{\normalfont\linespread{1.2}\fontsize{11}{12}\selectfont}
\makeatother
\definecolor{linkblue}{RGB}{0,0,139}      
\definecolor{navy}{RGB}{0,0,128}          
\definecolor{royalblue}{RGB}{65,105,225}  
\definecolor{steelblue}{RGB}{70,130,180}  
\definecolor{dodgerblue}{RGB}{30,144,255} 
\definecolor{mediumblue}{RGB}{0,0,205}    
\definecolor{darkslateblue}{RGB}{72,61,139} 
\usepackage[colorlinks = true,
            linkcolor = linkblue,
            urlcolor  = royalblue,
            citecolor = brown,
            anchorcolor = blue]{hyperref}

\usepackage[misc]{ifsym}
\usepackage{minitoc}
\usepackage{cleveref} 
\usepackage{subcaption}
\usepackage{booktabs}
\usepackage{amsmath}
\usepackage{mathtools}
\usepackage{amssymb}
\usepackage{graphicx}   
\usepackage{booktabs}   
\usepackage{algorithmic} 
\usepackage{algorithm}
\usepackage{tcolorbox}
\usepackage{verbatim}   
\usepackage{makecell}   
\usepackage{array}      
\usepackage{multirow}
\usepackage{xurl}
\usepackage{amsthm}
\usepackage{booktabs}
\usepackage{adjustbox}
\usepackage{caption}

\usepackage{listings}
\tcbuselibrary{listings,skins,breakable}

\lstdefinestyle{prompt}{
  basicstyle=\ttfamily\footnotesize,
  breaklines=true,           
  breakatwhitespace=false,   
  columns=fullflexible,      
  keepspaces=true,           
  showstringspaces=false,
  postbreak=\mbox{\textcolor{gray}{$\hookrightarrow$}\space} 
}

\newtheorem{remark}{Remark}
\newtheorem{proposition}{Proposition}

\newcommand{\E}{\mathbb{E}}

\newcommand{\numcell}[2][3.6em]{%
  \makebox[#1][r]{#2\hspace{1.8em}}%
}
\newcommand{\posdelta}[1]{%
  \makebox[0pt][r]{\smash{\raisebox{-0.25ex}{\scriptsize\textcolor{ForestGreen}{$+#1$}}}}%
}

\newcommand{\twocol}[1]{\multicolumn{2}{c|}{#1}}   
\newcommand{\twocolend}[1]{\multicolumn{2}{c}{#1}} 
\newcommand{\groupheader}[1]{%
  \specialrule{0pt}{0.8ex}{0pt}
  \multicolumn{11}{c}{#1}\\
  \specialrule{0.4pt}{0.6ex}{0pt}
}
\newcommand{\Lcell}[1]{\multicolumn{1}{l}{#1}}   
\newcommand{\Lcellbar}[1]{\multicolumn{1}{l|}{#1}} 
\newcommand{\dlmsubhead}{%
  & \Lcell{Static} & \Lcellbar{Dynamic}
  & \Lcell{Static} & \Lcellbar{Dynamic}
  & \Lcell{Static} & \Lcellbar{Dynamic}
  & \Lcell{Static} & \Lcellbar{Dynamic}
  & \Lcell{Static} & \Lcell{Dynamic}\\
  \cline{2-11}
}

\uselogo{} 

\title{Revolutionizing Reinforcement Learning Framework for Diffusion Large Language Models}

\correspondingauthor{yangling0818@163.com}

\author[1,2*]{Yinjie Wang}
\author[1*\ \Letter]{Ling Yang}
\author[1]{Bowen Li}
\author[1]{Ye Tian}
\author[ ]{Ke Shen}
\author[1]{Mengdi Wang}

\affil[*]{Equal contributions}
\affil[ \Letter]{Corresponding author}
\affil[1]{Princeton University}
\affil[2]{University of Chicago}

\begin{abstract}

{\fontsize{12pt}{12pt} \selectfont \raisebox{-0.06em}{\includegraphics[height=1em]{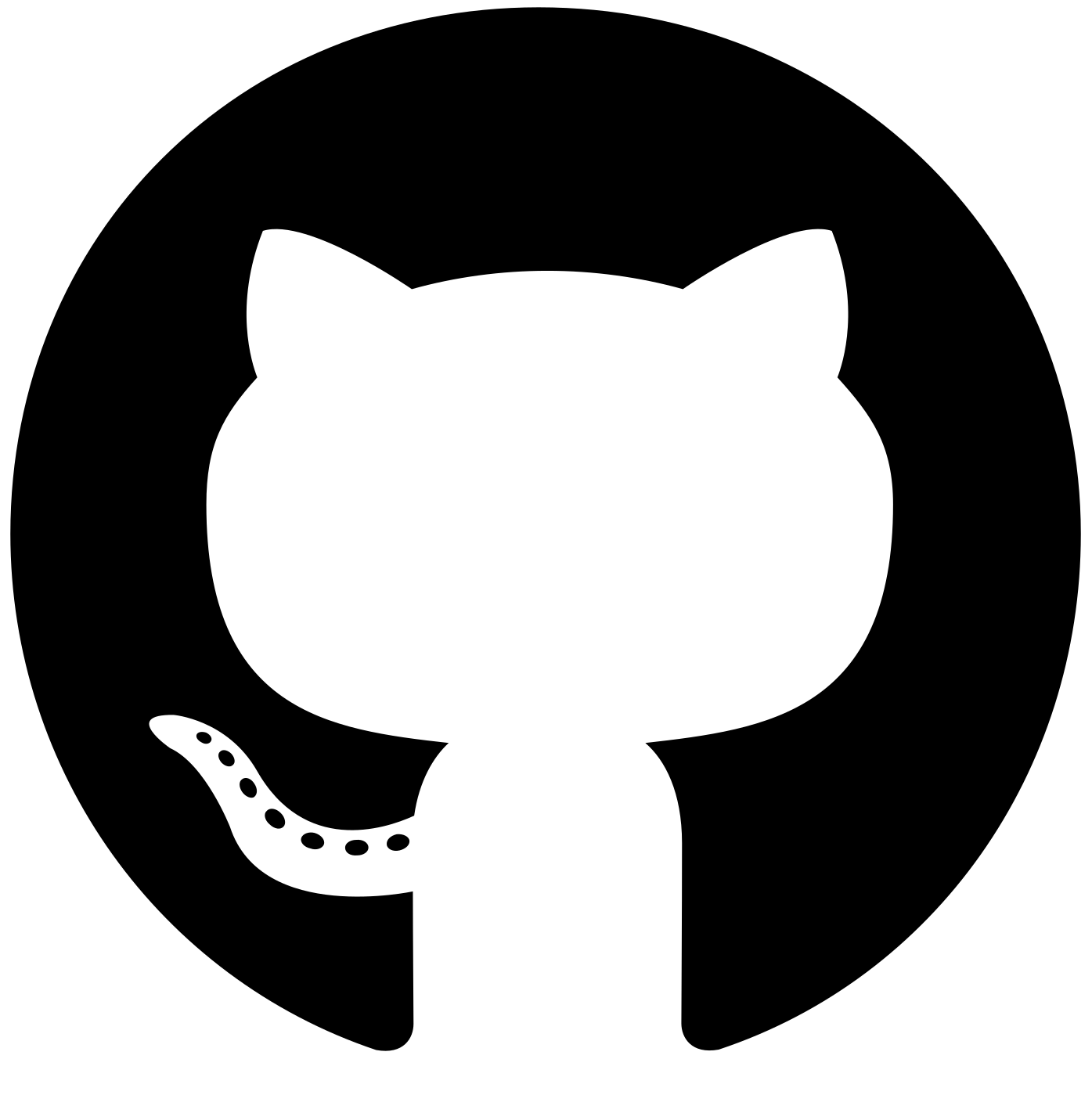}} Code: \href{https://github.com/Gen-Verse/dLLM-RL}{https://github.com/Gen-Verse/dLLM-RL} \quad \raisebox{-0.06em}{\includegraphics[height=1em]{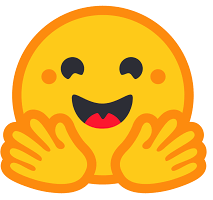}} Models: \href{https://huggingface.co/collections/Gen-Verse/trado-series-68beb6cd6a26c27cde9fe3af}{TraDo-4B/8B}}
\\
\\
We propose \textbf{TraceRL}, a trajectory-aware reinforcement learning framework for diffusion language models (DLMs) that incorporates preferred inference trajectory into post-training, and is applicable across different architectures. Equipped with a diffusion-based value model that enhances training stability, we demonstrate improved reasoning performance on complex math and coding tasks. Besides, it can also be applied to adapt block-specific models to larger blocks, which improves sampling flexibility. Employing TraceRL, we derive a series of state-of-the-art diffusion language models, namely \textbf{TraDo}. Although smaller than 7B-scale AR models, TraDo-4B-Instruct still consistently outperforms them across complex math reasoning tasks. TraDo-8B-Instruct achieves relative accuracy improvements of 6.1\% over Qwen2.5-7B-Instruct and 51.3\% over Llama3.1-8B-Instruct on mathematical reasoning benchmarks. Through curriculum learning, we also derive the first long-CoT DLM, outperforming Qwen2.5-7B-Instruct on MATH500 with an 18.1\% relative accuracy gain.
To facilitate reproducible research and practical applications, we release a comprehensive open-source framework for building, training, and deploying diffusion LLMs across diverse architectures. The framework integrates accelerated KV-cache techniques and inference engines for both inference and reinforcement learning, and includes implementations of various supervised fine-tuning and RL methods for mathematics, coding, and general tasks. 
\end{abstract}

\begin{document}

\maketitle

\begin{figure}[h]
  \centering
  
  \newlength{\panelht}
  \setlength{\panelht}{0.25\textheight}

  \begin{minipage}[t]{0.5\textwidth}
    \centering
    \vspace*{0pt}
    \includegraphics[height=\panelht]{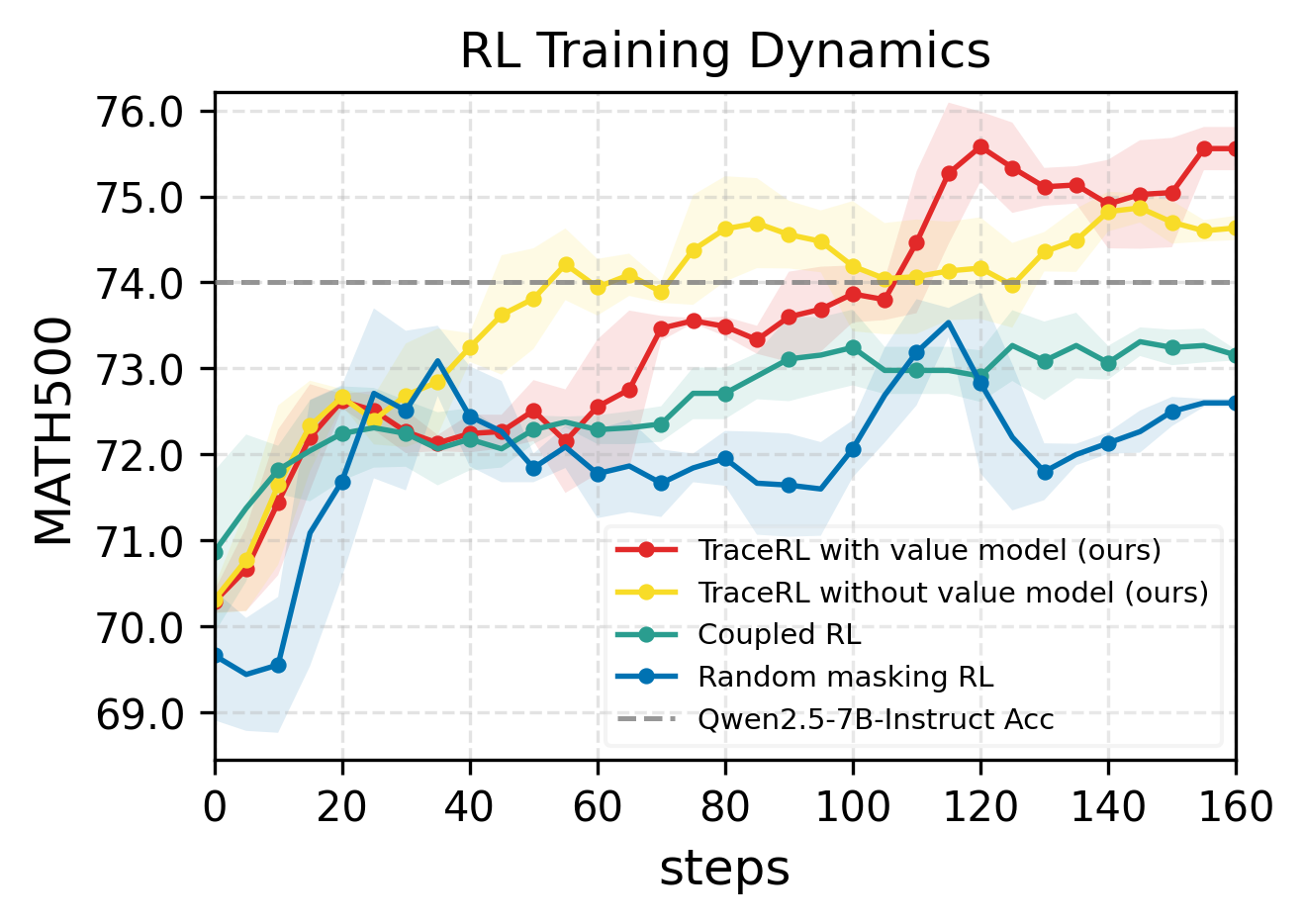}
  \end{minipage}
  \hfill
  \begin{minipage}[t]{0.47\textwidth}
  \centering
  \vspace*{2\baselineskip}
  \setlength{\tabcolsep}{4pt}     
  \renewcommand{\arraystretch}{1.05} 
  \footnotesize                   
  \begin{tabular}{@{}lrrr@{}}      
    \toprule
    \textbf{Model} & \textbf{MATH500} & \textbf{AIME2024} & \textbf{LCB-V2} \\
    \midrule
    \rowcolor[gray]{.9}\textbf{TraDo-8B-Thinking} & \textbf{87.4} & \textbf{35.5} & \textbf{34.6} \\
    \rowcolor[gray]{.9}\textbf{TraDo-8B-Instruct} & \textbf{78.5} & \textbf{13.3} & \textbf{25.9}\\
    \rowcolor[gray]{.9}\textbf{TraDo-4B-Instruct} & \textbf{75.6} & \textbf{10.3} & \textbf{18.7}\\
    Qwen2.5-7B-Instruct & 74.0 & 8.2 & 26.9\\
    Llama3.1-8B-Instruct & 51.9 & 6.7 & 20.0\\
    Dream-7B-Instruct & 38.7 & 0.0 & 10.7\\
    LLaDA-8B-Instruct & 38.3 & 1.7 & 5.9\\
    \bottomrule
  \end{tabular}
\end{minipage}
  \caption{\textit{Left}: RL training dynamics with different methods, where our TraceRL achieves the best optimization. \textit{Right}: Benchmark results on complex math reasoning tasks and LiveCodeBench-V2, all evaluated with KV-cache. TraDo series models outperform strong AR models even at smaller scales.}
\end{figure}

\section{Introduction}

To apply diffusion models \citep{ho2020denoising, song2019generative} on language tasks, approaches like projecting discrete tokens into a continuous
latent space \citep{graves2023bayesian, dieleman2022continuous, gulrajani2023likelihood}, and employ state transition matrices to directly derive the diffusion process \citep{austin2021structured, sahoo2024simple, shi2024simplified, ou2024your}, have been explored. Masked diffusion
models (MDMs) have emergyed as promising and scalable architecture for diffusion language model (DLM) recently \citep{nie2025large, dream2025}.
DLMs enable massive inference acceleration by parallel generation \citep{labs2025mercury, google_deepmind_gemini_diffusion_2025, song2025seed}, and improve consistency through bidirectional attention \citep{ye2024beyond, zhang2023planner}.

Diffusion language models with full attention have been explored in coding tasks \citep{labs2025mercury, google_deepmind_gemini_diffusion_2025, gong2025diffucoder, dreamcoder2025}, fixed-format reasoning tasks \citep{ye2024beyond, dream2025}. Beyond the full attention mechanism, block-attention diffusion models \citep{arriola2025block} have been scaled to large language models that excel at complex reasoning tasks \citep{JetAstra2025}. However, a unified and effective reinforcement learning (RL) framework across different model architectures remains underexplored.

Existing post-training frameworks for full-attention DLMs estimate the score of the whole sequence by adding random masks, which causes a mismatch with the optimal inference process \citep{dream2025, nie2025large, gong2025diffucoder, yang2025mmada}, since the sequential and logical nature of language is not purely random \citep{arriola2025block, gong2025diffucoder}. Block diffusion employs a semi-autoregressive supervised fine-tuning method \citep{arriola2025block}, while its reinforcement learning remains unexplored.
In this work, we demonstrate the importance of aligning the inference trajectory with the training objective, and propose a trajectory-aware reinforcement learning method, TraceRL, along with a diffusion-based value model. Our approach can be applied to both full-attention and block-attention models, achieving fast optimization.

Solely trained with TraceRL, we obtain 4B- and 8B-sized state-of-the-art diffusion language instruction models on reasoning tasks, both surpassing strong autoregressive models on math reasoning benchmarks. Furthermore, we develop the first long-CoT diffusion language model, TraDo-8B-Thinking, through a combination of TraceRL and long-CoT SFT.

We summarize our main contributions as follows:
\begin{itemize}
    \item We highlight the importance of aligning the training objective with the sampling trajectory. Building on this insight, we propose TraceRL, a fast and versatile reinforcement learning method applicable across diverse DLM architectures, together with a diffusion-based value model that reduces variance and improves training stability. We also explore TraceRL’s broader potential applications, including increasing the model’s block size and accelerating inference.
    \item Using TraceRL, we achieve significant improvements in reasoning ability across diverse tasks, resulting in the TraDo series. TraDo-4B-Instruct consistently outperforms strong AR models on various math reasoning datasets, and TraDo-8B-Instruct achieves relative accuracy gains of \textbf{6.1\%} over Qwen2.5-7B-Instruct and \textbf{51.3\%} over Llama3.1-8B-Instruct.
    \item Through curriculum learning combined with TraceRL and long-CoT SFT, we obtain the first long-CoT DLM, achieving relative accuracy gains of \textbf{18.1\%} on MATH500 and \textbf{28.6\%} on LiveCodeBench-V2 over Qwen2.5-7B-Instruct, demonstrating that diffusion language models can also perform complex reasoning.
    \item We release a complete and integrated framework for building, training, and deploying diffusion LLMs across diverse architectures. The framework incorporates implementations of various post-training methods and accelerated KV-cache techniques, enabling both reproducible research and practical applications.
\end{itemize}

\section{Preliminaries}
\label{prelim}

\subsection{Masked Diffusion Language Models with Full Attention}

Masked diffusion language models randomly replace non-masked tokens in the raw sample \(x_0\) with mask tokens \([{\rm MASK}]\), resulting in \(x_t\), where \(t \in [0, 1]\). This process can be formulated as:
\begin{align}
\label{process}
q(x_t \mid x_0) = \prod_{i=1}^{n} \mathrm{Cat}\!\left(x_t^i; (1 - t) \delta_{x_0^i} + t \delta_{\text{[MASK]}}\right).
\end{align}
It has been demonstrated that training objective of DLM can be derived from data likelihood \citep{shi2024simplified, sahoo2024simple, zheng2024masked, ou2024your}. Specifically, the optimization objective is the evidence lower bound of $\log p_{\theta} (x)$:
\begin{align}
\label{dlmobj}
\mathcal{J}_{full}(x_0, Q, \theta) = \int_{0}^{1} \frac{1}{t |x_0|} \, 
\mathbb{E}_{q(x_t \mid x_0)} \left[ 
\sum_{i: x_t^i = \text{[MASK]}} 
\log p_{\theta}(x_0^i \mid x_t, Q) 
\right] dt,
\end{align}
where $Q$ stands for the prompt part, and $|x_0|$ stands for the number of tokens in $x_0$.

\subsection{Adapted From AR Models}

Besides training from scratch with this objective \citep{nie2025large,yang2025mmada}, adapting from an auto-regressive model has been shown to be a promising approach for preserving the AR model's capabilities \citep{dream2025}, while only requiring adaptation training. 
This approach is more efficient than training from scratch and has led to a series of powerful DLMs \citep{dream2025, gong2025diffucoder, dreamcoder2025}. Its theoretical supervised fine-tuning objective is the same as Equation~\ref{dlmobj}; the only difference is that a right-shift of the logits output is required, since its base model is an autoregressive LLM.

\subsection{Block Diffusion}

Block diffusion models \citep{arriola2025block, JetAstra2025} employ a block-diffusion attention mechanism that combines the training efficiency of autoregressive models with the sampling efficiency of diffusion models. Moreover, their block-wise generation makes KV-cache support naturally available, whereas full-attention DLMs experience performance degradation when using KV-cache \citep{hu2025accelerating, liu2025dllm, ma2025dkv, arriola2025block, yu2025dimple, wu2025fast}.
A key limitation of block diffusion lies in its fixed block size. When the block size is small, its sampling speed potential falls short of that of full-attention DLMs, since fast sampling fundamentally depends on generating a large number of tokens at once \citep{song2025seed}. This highlights the need to adapt block diffusion to larger block sizes in order to achieve sampling flexibility.

\subsection{Decoding Strategy and Accelerated Inference}
\label{sampleintro}

The reverse process (Equation~\ref{process}) can be accelerated through an iterative generation procedure, in which multiple masked tokens are approximately recovered in a single step when transitioning from a noise level $t$ to a lower level $s < t$. Specifically, given a masked sequence $x_t$, we select [MASK] tokens to be unmasked based on the confidence at each position $i$, defined as $\max_{x_s^i} p_{\theta}(x_s^i \mid x_t, Q)$. A common strategy is to unmask either a fixed number of tokens with the highest confidence at each step \citep{chang2022maskgit, kim2025train}, referred to as \textbf{static sampling}, or all positions whose confidence exceeds a predefined threshold $\mathcal{T}$ \citep{yu2025dimple, wu2025fast}, referred to as \textbf{dynamic sampling}, with the latter often being faster. By iteratively applying this procedure starting from a fully masked sequence, we eventually obtain a fully unmasked sequence.
This parallel decoding strategy, combined with KV-cache techniques \citep{hu2025accelerating, liu2025dllm, ma2025dkv, arriola2025block, yu2025dimple, wu2025fast}, gives DLMs faster sampling speeds than LLMs \citep{labs2025mercury, google_deepmind_gemini_diffusion_2025, song2025seed}.

\section{Mismatch Between Post-Training Objective and Inference Trajectory}
\label{mismatch}

Pretraining with the fully random masking objective (Equation~\ref{dlmobj}) enables parallel decoding. However, language inherently depends on previous context. When combined with the chosen decoding strategy and the widely adopted block-wise generation with KV-cache, this creates a mismatch between the post-training objective and the model’s inference behavior. In this section, we present a simple experiment to illustrate this discrepancy.

\begin{figure}[t!]
  \centering
  \includegraphics[width=0.99\textwidth]
  {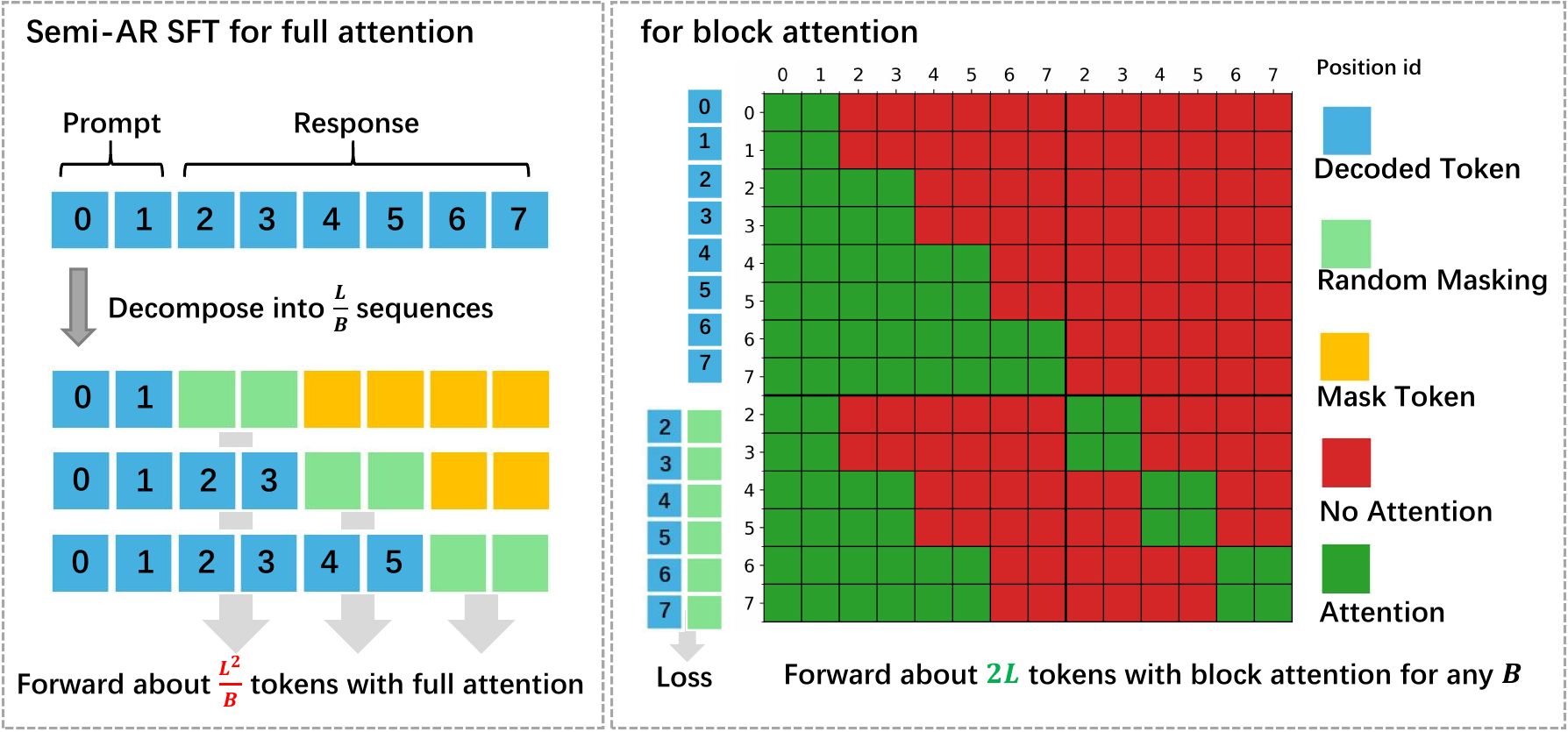}
  \caption{Semi-AR SFT for full attention and block attention model Overview. This example uses a block size of $B = 2$ and a sequence length of $L = 6$. Block diffusion models naturally adapt to semi-AR SFT efficiently, whereas full attention models require slicing the data by $B$.}
  \label{sft}
\end{figure}

\subsection{Semi-Autoregressive Fine-Tuning}

Block diffusion models naturally employ semi-autoregressive fine-tuning with the following objective, which trains the language model’s ability to generate later tokens conditioned on earlier context using block attention, while preserving the sampling efficiency characteristic of DLMs.
\[
\mathcal{J}_{\text{semi}}(x, Q, \theta) = \sum_{i = 1}^{\lceil L / B \rceil} \mathcal{J}_{\text{full}}\big(x^{(i - 1)B:\min(iB, L)}, [Q, x^{0:(i - 1)B}], \theta\big).
\]
For full-attention DLMs, applying the semi-autoregressive objective inevitably increases the number of forward passes by approximately $\lceil L / B \rceil$ (see Figure~\ref{sft}). Nevertheless, the optimization performance of $\mathcal{J}_{\text{semi}}(x, Q, \theta)$ is substantially better than that of $\mathcal{J}_{\text{full}}(x, Q, \theta)$ under $\lceil L / B \rceil$ independent repetitions, where the computational load is the same (see Table~\ref{tracesft}). This demonstrates the importance of aligning the post-training objective with the general left-to-right inference pattern.

\subsection{Aligning Post-training with Preferred Inference Traces}

To further investigate, we collect the pre-trained model's own preferred inference traces from the data and use them for finetuning. Specifically, we apply static sampling (see definition in Section~\ref{sampleintro}) to obtain the optimal trace for each data point, which is then used in our finetuning (see Section~\ref{misdetail} for details).

\begin{table}[t]
  \centering
    \caption{We explore how effectively different methods tune the model to learn CoT reasoning in order to improve reasoning accuracy under non-CoT prompts. The 2000 datapoints are generated by Qwen2.5 models and filtered to retain high quality. $l$ denotes the length of each collected step in the whole trace. ``$\times m$" indicates that we apply $m$ independent random maskings to augment the data scale for a fair comparison. ``Token forward'' denotes the number of tokens processed by the model, representing the computational load or time. ``Token trained'' refers to the number of tokens actually used in the optimization objective. We report accuracy on MATH500 here. The block-attention model used here is SDAR-4B-Chat (default block size of $4$), and the full-attention model is Dream-7B-Instruct.}
  \resizebox{0.95\textwidth}{!}{%
    \begin{tabular}{lccccccccc}
      \toprule
      \multirow{4}{*}{Token Utilization Efficiency}& \multicolumn{5}{c}{Full Attention} & & \multicolumn{3}{c}{Block Attention} \\
      \cmidrule(lr){2-6} \cmidrule(lr){8-10}
      & trace & \multicolumn{2}{c}{semi-ar ($B=16$)} & \multicolumn{2}{c}{fully random} & & trace & \multicolumn{2}{c}{semi-ar ($B=4$)} \\
      \cmidrule(lr){2-2} \cmidrule(lr){3-4} \cmidrule(lr){5-6} \cmidrule(lr){8-8} \cmidrule(lr){9-10}
      & $l=16$ & $\times 2$ & $\times 1$ & $\times 35$ & $\times 1$ & & $l=2$ & $\times 2$ & $\times 1$ \\
      \midrule
      accuracy (\%)        & 54.4 & 53.4 & 52.6 & 45.1 & 39.6 & & 71.3 & 70.4 & 70.0 \\
      \# token forward (M) & 39.2 & 78.4 & 39.2 & 39.2 & 1.1 & & 4.5 & 4.5 & 2.2 \\
      \# token trained (M) & 1.1 & 1.1 & 0.6  & 20.0 & 0.6 & & 1.1 & 1.1 & 0.6 \\
      \bottomrule
    \end{tabular}%
  }

  \label{tracesft}
\end{table}

Table~\ref{tracesft} shows that using the model’s own preferred inference traces achieves optimal performance over the baselines, even when the computational load is equal to or lower than that of the baselines, for both block-attention and full-attention structures.

\subsection{Reinforcement Learning Naturally Leverages Inference Traces}

A major limitation of this fine-tuning approach is that collecting inference traces requires substantial effort. In contrast, reinforcement learning naturally produces these traces during rollouts, making it a more practical and effective strategy for post-training. Accordingly, we propose TraceRL in the following section.

\section{RL Training with Trajectory}

Current RL methods for DLMs focus on full-attention models, rewarding or penalizing rollout responses based on the overall generated sequence \citep{zhao2025d1, yang2025mmada, gong2025diffucoder} through the random masking objective $\mathcal{J}_{\text{full}}$ (Equation~\ref{dlmobj}). 
We propose \textbf{TraceRL} (Figure~\ref{rl}, detailed pipeline in \cref{algapp}), which instead focuses on the intermediate traces generated by the DLM, and can be applied across different architectures. We also introduce a diffusion-based value model, which helps reduce variance and improve stability in training. We demonstrate that this leads to a more effective optimization.

For each generated response $\tau_i$ given the task $Q$, we can write it in the trajectory form 
$\tau_{i} \stackrel{\triangle}{=} (\tau_{i}(1), \dots, \tau_{i}(|\tau_{i}|))$, 
where $|\tau_{i}|$ is the total number of decoding steps, and $\tau_{i}(t)$ is the set of tokens decoded during the $t$-th step. 
TraceRL rewards or penalizes the generation trajectory of policy $\pi_{\theta}$, based on the verifiable reward $r_i$ for the response $\tau_i$. Process-level rewards can also be incorporated through the use of a value model.

\begin{figure}[t!]
  \centering
  \includegraphics[width=0.97\textwidth]
  {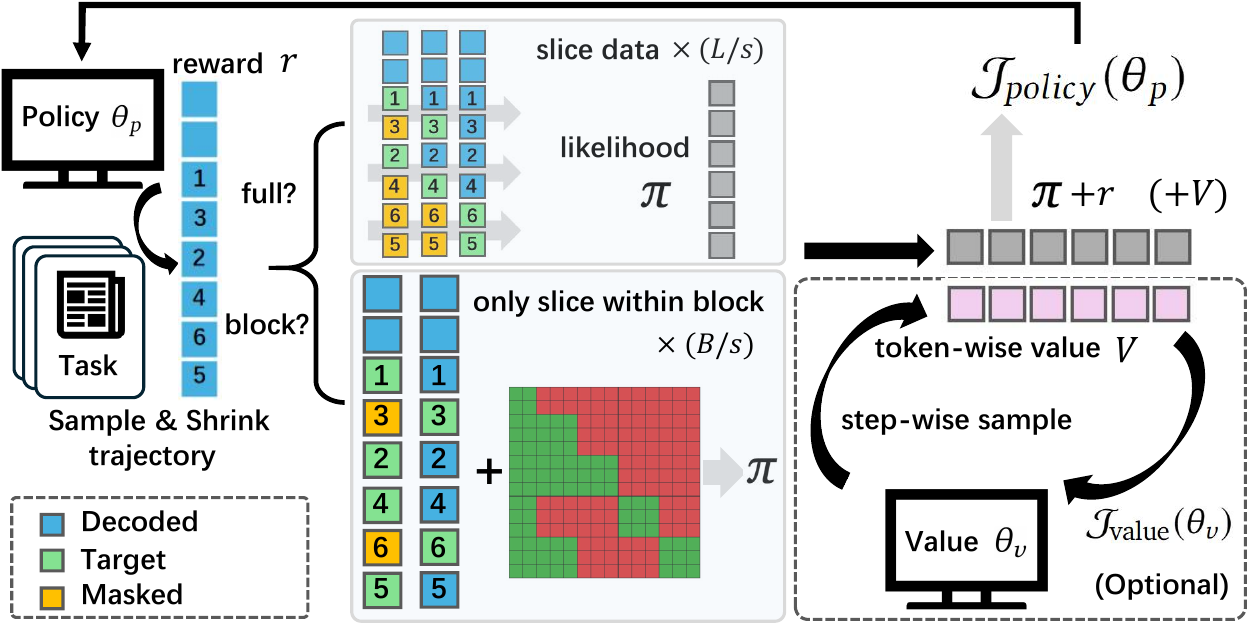}
  \caption{TraceRL Overview (detailed pipeline in \cref{algapp}). This is an example for $s = 2$, $L = 6$ and $B = 3$. We aggregate every $s$ neighboring steps to perform trajectory-aware reinforcement learning. The integers in these squares represent the sequence of the policy’s inference process.}
  \label{rl}
\end{figure}

\subsection{Accelerated Training with Shrinkage Parameter}

For full-attention models, decomposing the sequences based on each sampling step leads to a large number of forward passes during training. So we introduce a shrinkage parameter $s$, which aggregates every $s$ neighboring steps to improve training efficiency.
Formally, we shrink the trajectory $\tau_i$ into 
$\tau_i^s \stackrel{\triangle}{=} (\tau_{i}^s(1), \dots, \tau_{i}^s(|\tau_i^s|))$, where $\tau_i^s(k) \stackrel{\triangle}{=} \cup_{j = s(k - 1) + 1}^{\min( s k, |\tau_i|)} \tau_i(j)$ and $ |\tau_i^s| = \lceil |\tau_{i}|/s \rceil$,
then optimize the following objective:
\begin{align}
\label{objpolicy}
\mathcal{J}_{policy}(\theta_p) = \E_{\substack{Q \sim D_{\text{task}}\\\{\tau_i \}_{i = 1}^G \sim \pi_{\text{old}} (\cdot \mid Q)}} \Big[ \sum_{i = 1}^G \sum_{t = 1}^{ |\tau^s_i|} \sum_{o_k \in \tau_{i, t}^s} C_{\epsilon}( \frac{\pi_{\theta_p} (o_k \mid \tau_{i}^s(1:(t-1)))}{\pi_{\text{old}} (o_k \mid \tau_{i}^s(1:(t-1)))}, A_i) / |\tau_{i}^s(t)| \Big] - \beta \mathbb{KL}[\pi_{\theta}||\pi_{\text{old}}],
\end{align}
where $C_{\epsilon}(r, A) \stackrel{\triangle}{=} \min (r A, \text{clip} (r, 1 - \epsilon, 1 + \epsilon)A )$, $\tau_i^s(1:t) \stackrel{\triangle}{=} \cup_{j = 1}^t \tau_i^s(j)$, $\pi_{\text{old}}$ is the old policy, and $A_i$ is the standardized reward.

By introducing the shrinkage parameter $s$, we reduce the training computation complexity by a factor of $s$. We can directly use $\mathcal{J}_{policy}(\theta_p)$ for efficient RLVR (reinforcement learning with verifiable rewards) training.

\subsection{Value Model with Diffusion Modeling}

Rather than assigning a single sequence-level advantage to all tokens in a completion, a value function enables prefix-conditioned, token-wise advantages (e.g., GAE-based estimates), providing a variance-reducing baseline that typically stabilizes policy optimization \citep{schulman2017proximal, hu2025open}. We adopt a diffusion-based value model to estimate step-wise values.
For each rolled-out trajectory \(\tau\) (possibly shrunken by a shrinkage factor \(s\)), we keep the notation \(\tau\) for brevity. We evaluate a frozen value network \(V_{\text{old}}\) to obtain token-wise values along the trajectory: $V^{\text{old}}_j \triangleq\big(V_{\text{old}}(\tau)\big)_j$, $j\in \tau .$
Concretely, for each trace step \(t\), we condition on the prefix \(\tau(1{:}t-1)\) and predict the values for all tokens generated at that step, i.e.,$
\{\,V^{\text{old}}_j : j \in \tau(t)\,\}.$
We then define the step-wise value at step \(t\) as $V_{t}^{\star,\text{old}} = \sum_{j\in \tau(t)} V^{\text{old}}_j / |\tau(t)|.$
\(V^{\text{old}}\) is treated as a \emph{stop-gradient} baseline when constructing returns and advantages; the learnable value network \(V_{\theta_v}\) is only updated via its own regression objective later.

We now derive the training objective. Given token-wise rewards \(r_j\), we first form the step-wise reward \(r_t^{\star} = \sum_{j \in \tau(t)} r_j / |\tau(t)|\). We compute step-wise returns backward as \(R_t^{\star} = r_t^{\star} + \gamma R_{t+1}^{\star}\) for \(t\le|\tau|\) with terminal \(R_{|\tau| + 1}^{\star} = 0\). Using the frozen baseline, define \(V_t^{\star,\text{old}} = \sum_{j \in \tau(t)} V_j^{\text{old}} / |\tau(t)|\), the TD residual \(\delta_t^{\star} = r_t^{\star} - V_t^{\star,\text{old}} + \gamma V_{t+1}^{\star,\text{old}}\), and the step-wise GAE \(A_t^{\star} = \delta_t^{\star} + \gamma \lambda A_{t+1}^{\star}\). We then map step-level quantities back to tokens via \(R_j = r_j + \gamma R_{t_j+1}^{\star}\) and \(A_j = (r_j - V_j^{\text{old}}) + \gamma V_{t_j+1}^{\star,\text{old}} + \gamma \lambda A_{t_j+1}^{\star}\), where \(t_j\) is such that \(j \in \tau(t_j)\). The value network is trained with a clipped regression loss:
\begin{align}
\label{objvalue}
\mathcal{J}_{\text{value}}(\theta_v) = \tfrac{1}{2}\,\mathbb{E}_{\tau}\Big[\tfrac{1}{|\tau|}\sum_{j \in \tau} \max\big((V_{\theta_v}(\tau)_j - R_j)^2,\; (V_j^{\text{clip}} - R_j)^2\big)\Big],
\end{align}
where \(V_j^{\text{clip}} = V_j^{\text{old}} + \mathrm{clip}\,\!\big(V_{\theta_v}(\tau)_j - V_j^{\text{old}},\, -\epsilon,\, \epsilon\big)\). The token advantages \(A_j\) are used in \(\mathcal{J}_{\text{policy}}(\theta_p)\) to update the policy. We provide the explicit forms of $R_j$ and $A_j$ in Proposition~\ref{prop:token-from-trace} (see Section~\ref{proofsect} in Appendix).

\subsection{Sliced Training in Block Diffusion}

Block diffusion employs a block-attention mechanism for efficient supervised fine-tuning, and we extend this advantage to reinforcement learning. For each derived and processed trace $\tau$ generated by block-wise inference, we represent it as $\tau = (b_1, \dots, b_{\lceil |\tau| / B \rceil})$, where each block is defined as $b_k = (\tau_{k,1}, \dots, \tau_{k, |b_k|})$ with $|b_k|$ steps. We denote the block size as $B$. The training objective can then be sliced from $\sum_{i=1}^{|\tau|} f(\tau(i))$ into a $B'$ training slice, $\left(\sum_{k=1}^{\lceil |\tau|/ B \rceil} f(\tau_{k,l}) \mathbf{1}_{\{l \le |b_k|\}}\right)_{l = 1}^{B'}$, where $f$ is a task-specific function and $B' = \max_k\{\, |b_k|\,\} \le \lceil B / s \rceil$. Each slice only needs to be forwarded once using block attention (Figure~\ref{rl}). This formulation maximizes the utility of the block-attention mechanism and enables highly parallel and efficient training, applicable to both policy and value model training, which is significantly more efficient than full-attention training.

\section{Experiments}

In this section, we demonstrate the superiority and broad applicability of TraceRL. We present evaluation results for our state-of-the-art 4B and 8B instruction models trained solely with TraceRL, as well as for the long-CoT model TraDo-8B-Thinking, trained by a combination of TraceRL and SFT. We also highlight interesting applications such as block size enlargement and the observation of acceleration effects.

\subsection{Experimental Setups}

\subsubsection{Data}

We use different data source for reinforcement learning. For the Math tasks, we choose the MATH training dataset \citep{hendrycks2021measuring} and retain only level 3–5 tasks \citep{hu2025open}, resulting in 8K hard tasks. For the coding RL setting, we use 6K verified problems from PrimeIntellect \citep{jaghouar2024intellect}, verified by DeepCoder \citep{deepcoder2025}.

Our evaluation focuses on reasoning tasks in mathematics and coding. For mathematics, we use GSM8K \citep{cobbe2021training}, MATH500 \citep{hendrycks2021measuring}, and AIME \citep{AIME2024}. For coding, we use LiveCodeBench-V2 \citep{jain2024livecodebench} and LiveBench \citep{white2024livebench}.

\subsubsection{Models and Optimization}

Our experiments include both full-attention and block-attention models. For full attention, we use Dream-7B-Instruct \citep{dream2025} as the base model, and for block attention, we use the SDAR \citep{JetAstra2025} series of models, trained with a block size of 4.

We now describe our reinforcement learning settings. For the block-diffusion model, during each sampling step we sample 128 tasks and 32 responses per task, with a dynamic sampling strategy ($\mathcal{T}=0.9$) and temperature 1.0. For the full-attention model, during each RL sampling step we sample 56 tasks and generate 8 responses per task from the policy $\pi_{\theta}$ using KV-Cache, with a temperature of 1.0 and a static sampling strategy (one token per step) to improve rollout quality. During training, we set the learning rate of the policy to $1 \times 10^{-6}$, with $\epsilon = 0.2$ and $\beta = 0.01$. When using the value model, we set the learning rate for RL to $5 \times 10^{-6}$, and use $\gamma = \lambda = 1.0$ as default values. We provide additional optimization details in the appendix (Sections~\ref{appsample} and \ref{traindetailapp}).

\subsubsection{Evaluation}

During evaluation, we report results for both static and dynamic sampling. For the block-attention model, we use temperature $t = 1.0$ for dynamic sampling and greedy decoding (top-$k = 1$ for static sampling), with a default block size of $4$. For the full-attention model, we use temperature $t = 0.1$ with block size $4$ for static sampling and block size $32$ for dynamic sampling. See Appendix (Section~\ref{appeval} for evaluation details and Section~\ref{apptemp} for prompt templates).

\begin{table*}[t!]
\centering
\caption{The main benchmark results across different math and coding tasks. ``Static" refers to static sampling, and ``Dynamic" refers to dynamic sampling. The long-CoT model TraDo-8B-Instruct here is evaluated by dynamic sampling with threshold $0.9$.}
\label{maintable}
\begin{adjustbox}{max width=\textwidth}
\begin{tabular}{l|cc|cc|cc|cc|cc}
\specialrule{1.2pt}{0pt}{0pt}
\multirow{1}{*}{\textbf{Model}} &
  \multicolumn{2}{c|}{\textbf{MATH500}} &
  \multicolumn{2}{c|}{\textbf{AIME2024}} &
  \multicolumn{2}{c|}{\textbf{GSM8K}} &
  \multicolumn{2}{c|}{\textbf{LiveCodeBench-v2}} &
  \multicolumn{2}{c}{\textbf{LiveBench}} \\
\hline\hline
\groupheader{\textbf{Autoregressive Models}}
Llama3.1-8B-Instruct &
  \twocol{\numcell{51.9}} & \twocol{\numcell{6.7}} &
  \twocol{\numcell{84.5}} & \twocol{\numcell{20.0}} &
  \twocolend{\numcell{19.7}} \\
Qwen2.5-7B-Instruct &
  \twocol{\numcell{74.0}} & \twocol{\numcell{8.2}} &
  \twocol{\numcell{89.9}} & \twocol{\numcell{26.9}} &
  \twocolend{\numcell{31.1}} \\
\hline
\groupheader{\textbf{Diffusion Language Models}}
\dlmsubhead
LLaDA-8B-Instruct  & \numcell{37.3} & \numcell{38.3} & \numcell{0.5} & \numcell{1.7} & \numcell{82.5} & \numcell{82.5} & \numcell{5.9} & \numcell{5.5} & \numcell{4.9} & \numcell{6.0} \\
Dream-7B-Instruct  & \numcell{38.7} & \numcell{32.3} & \numcell{/} & \numcell{/} & \numcell{72.7} & \numcell{57.8} & \numcell{10.7} & \numcell{4.7} & \numcell{10.7} & \numcell{4.9} \\
\hline
SDAR-4B-Chat       & \numcell{70.2} & \numcell{67.4} & \numcell{5.0} & \numcell{8.2} & \numcell{90.2} & \numcell{88.9} & \numcell{15.6} & \numcell{11.2} & \numcell{14.0} & \numcell{7.6} \\
\rowcolor[gray]{.9}
\textbf{TraDo-4B-Instruct} &
  \numcell{\textbf{75.6}}\posdelta{5.4} & \numcell{\textbf{71.8}}\posdelta{4.4} &
  \numcell{\textbf{8.3}}\posdelta{3.3} & \numcell{\textbf{10.3}}\posdelta{2.1} &
  \numcell{\textbf{91.2}}\posdelta{1.0} & \numcell{\textbf{90.3}}\posdelta{1.2} &
  \numcell{\textbf{18.7}}\posdelta{3.1} & \numcell{\textbf{15.1}}\posdelta{3.9} & \numcell{\textbf{12.9}} & \numcell{\textbf{10.4}}\posdelta{2.8} \\
SDAR-8B-Chat       & \numcell{74.3} & \numcell{70.7} & \numcell{11.8} & \numcell{8.3} & \numcell{91.1} & \numcell{90.4} & \numcell{18.5} & \numcell{15.3} & \numcell{11.5} & \numcell{11.2} \\
\rowcolor[gray]{.9}
\textbf{TraDo-8B-Instruct} & \numcell{\textbf{78.5}}\posdelta{4.2} & \numcell{\textbf{75.5}}\posdelta{4.8} & \numcell{\textbf{13.3}}\posdelta{1.5} & \numcell{\textbf{11.0}}\posdelta{2.7} & \numcell{\textbf{92.3}}\posdelta{1.2} & \numcell{\textbf{91.2}}\posdelta{0.8} & \numcell{\textbf{25.9}}\posdelta{7.4} & \numcell{\textbf{22.4}}\posdelta{7.1} & \numcell{\textbf{22.7}}\posdelta{11.2} & \numcell{\textbf{20.6}}\posdelta{9.4}  \\
\rowcolor[gray]{.9}
\textbf{TraDo-8B-Thinking} &
  \multicolumn{2}{c|}{\numcell{\textbf{87.4}}\posdelta{13.1}} &
  \multicolumn{2}{c|}{\numcell{\textbf{35.5}}\posdelta{23.7}} &
  \multicolumn{2}{c|}{\numcell{\textbf{94.2}}\posdelta{3.1}} &
  \multicolumn{2}{c|}{\numcell{\textbf{34.6}}\posdelta{16.1}} &
  \multicolumn{2}{c}{\numcell{\textbf{36.0}}\posdelta{23.8}} \\
\specialrule{1.2pt}{0pt}{0pt}
\end{tabular}
\end{adjustbox}
\end{table*}

\subsection{Instruction Models Trained with TraceRL}

We obtain TraDo-4B-Instruct and TraDo-8B-Instruct by applying TraceRL to math and coding tasks, starting from the SDAR base model.  We evaluate our models across math and coding tasks, as well as five reasoning datasets, and compare them against both strong diffusion language models and autoregressive models \citep{yang2024qwen2, dubey2024llama}. Our instruction models achieve state-of-the-art performance on reasoning tasks among current diffusion models, demonstrating the effectiveness of TraceRL. Both dynamic (faster) and static (more accurate) sampling abilities have been improved significantly. Notably, TraDo-4B-Instruct outperforms strong autoregressive baselines such as Qwen2.5-7B-Instruct across all math tasks.

\subsection{Long-CoT Diffusion Language Model}

The long-CoT diffusion language model TraDo-8B-Thinking is derived from TraDo-8B-Instruct, trained through a combination of long-CoT SFT and TraceRL. As the first long-CoT diffusion language model, TraDo-8B-Thinking demonstrates strong reasoning capabilities across benchmarks, notably achieving 85.8\% accuracy on MATH500, showing that diffusion language models can also excel at complex reasoning tasks.

\begin{figure}[t!]
  \centering
  \includegraphics[width=0.95\textwidth]
  {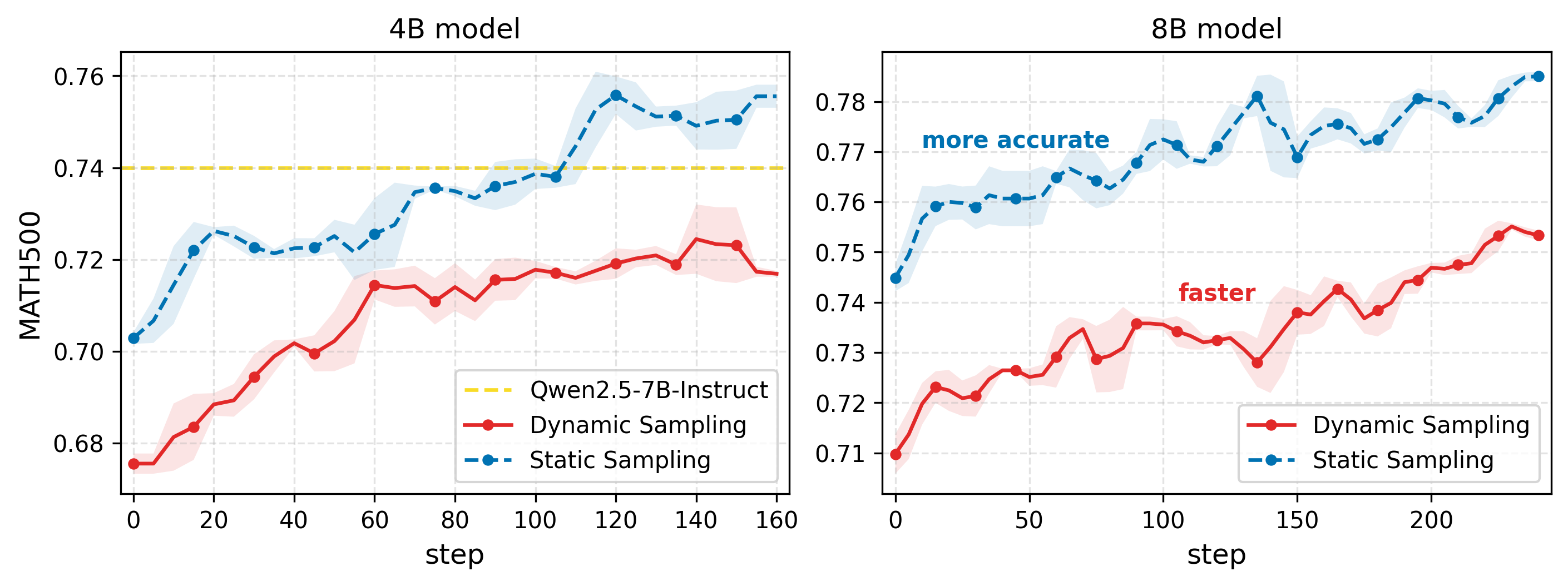}
  \caption{The TraceRL training curves for the 4B and 8B models on the math task. The red curve denotes the dynamic sampling accuracy, which achieves faster sampling speed, while the blue curve denotes the static sampling accuracy, which achieves higher accuracy. The 4B model is trained with a value model, whereas the 8B model is trained directly using $\mathcal{J}_{\text{policy}}$.}
  \label{4b8b}
\end{figure}

\subsection{Training Dynamics of TraceRL}

We record the training dynamics on math tasks for our 4B and 8B instruction models (see Figure~\ref{4b8b}). Although we adopt dynamic sampling during the RL training process, both dynamic and static accuracy improve steadily, and the trend suggests further potential for scaling. This RL training significantly enhances the models’ math reasoning ability: on MATH500, TraDo-4B-Instruct improves by 5.4\% (static) and 4.2\% (dynamic), surpassing Qwen2.5-7B-Instruct after optimization, while TraDo-8B-Instruct improves by 4.2\% (static) and 4.8\% (dynamic) (see Figure~\ref{4b8b}).

\begin{figure}[t!]
  \centering
  \includegraphics[width=0.95\textwidth]
  {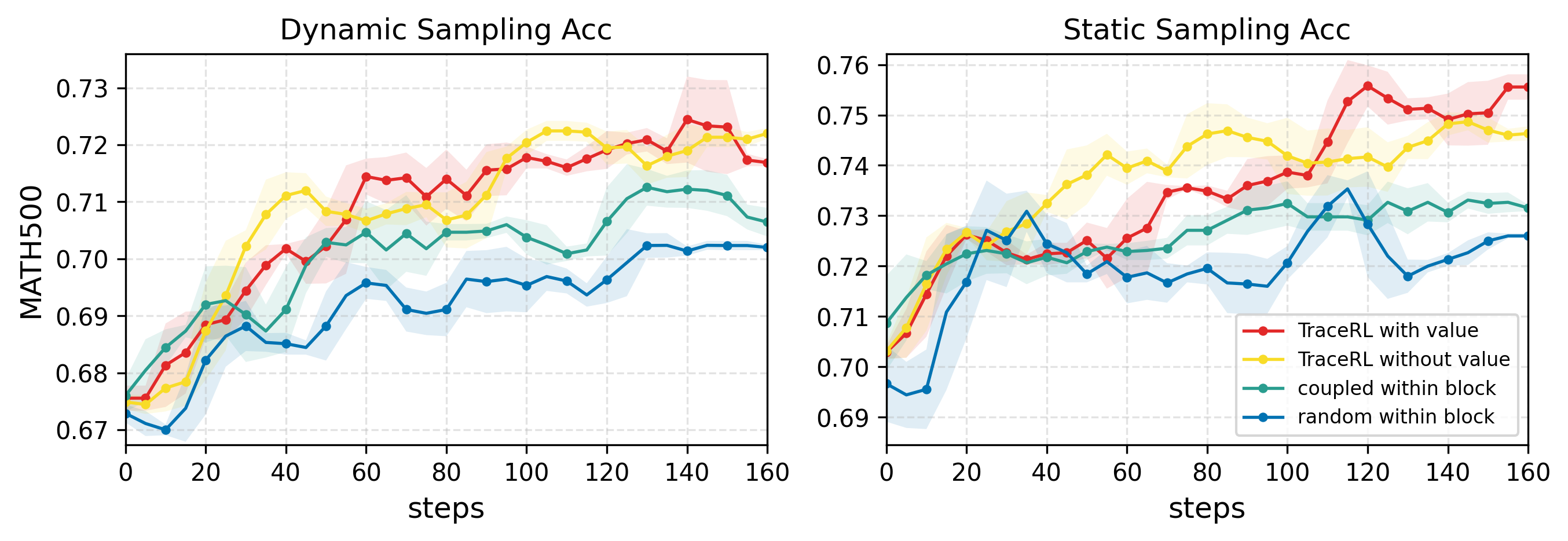}
  \caption{RL method ablations on block diffusion models and math RL tasks. The red and yellow curves represent TraceRL with and without a value model, respectively. The blue curve corresponds to training with a random masking objective within block, similar to the semi-autoregressive training approach. The green curve represents training with an additional complementary mask within block.}
  \label{rl_ablation_plot}
\end{figure}

\subsection{Stronger Optimization with TraceRL}

We compare TraceRL with existing RL methods, focusing first on block diffusion models. Although current RL methods are primarily developed for full attention models, we adapt them directly to the block setting. For the random masking approach \citep{yang2025mmada}, we restrict sampling within each block, making it resemble a semi-autoregressive method. For coupled RL \citep{gong2025diffucoder}, we introduce a complementary objective within each block, which provides more stable and effective training. We evaluate these methods on math tasks, with the results shown in Figure~\ref{rl_ablation_plot}. The experiments demonstrate that TraceRL achieves the best optimization performance, both with and without a value model. All RL methods use dynamic sampling during rollout, and TraceRL attains superior optimization under both dynamic and static evaluation settings. This demonstrates the importance of optimizing over the preferred trace, even within a small block.

\begin{figure}[t]
  \centering
  \begin{subfigure}{0.48\textwidth}
    \centering
    \includegraphics[width=\linewidth]{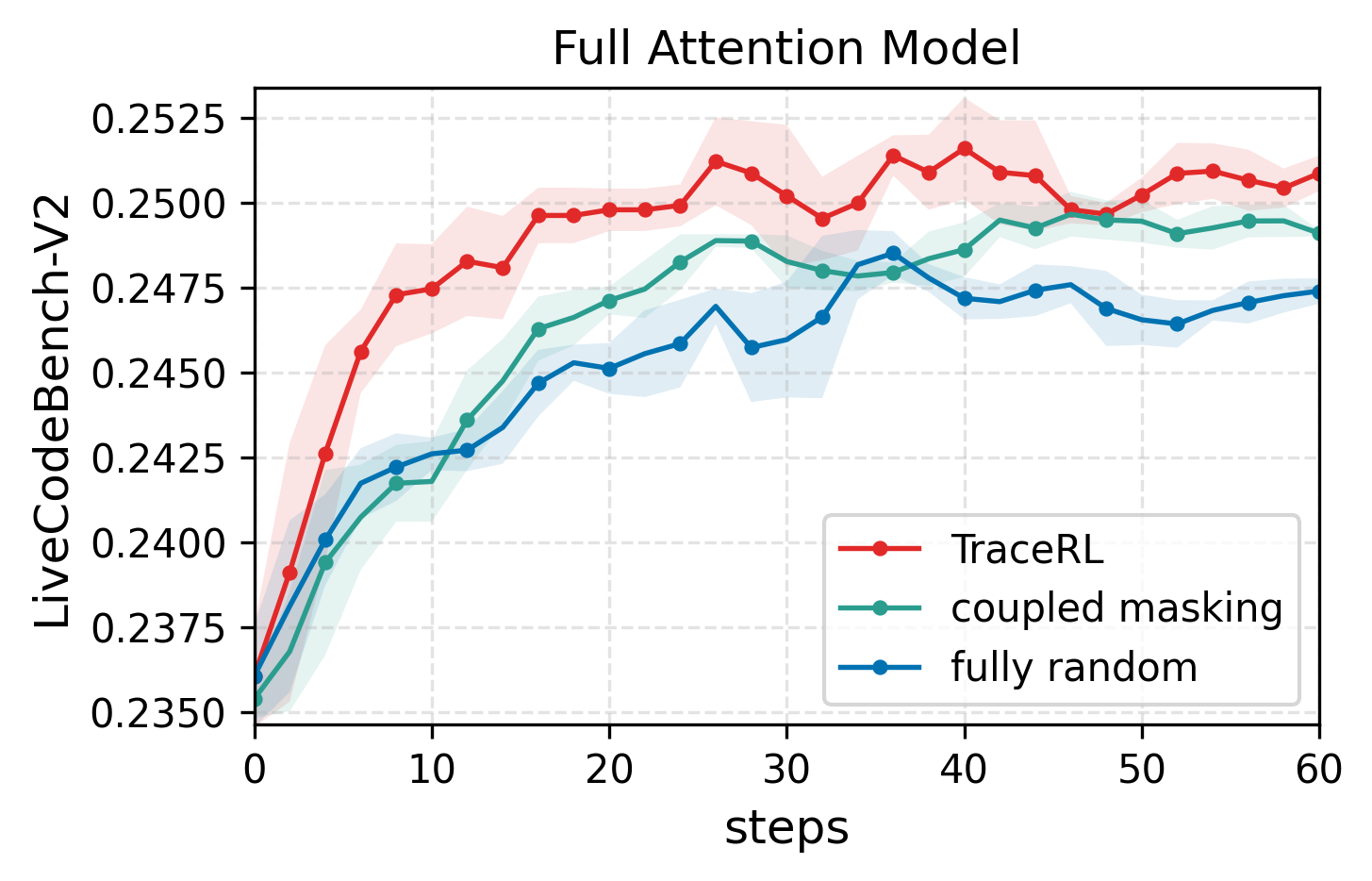}
  \end{subfigure}
  \begin{subfigure}{0.48\textwidth}
    \centering
    \includegraphics[width=\linewidth]{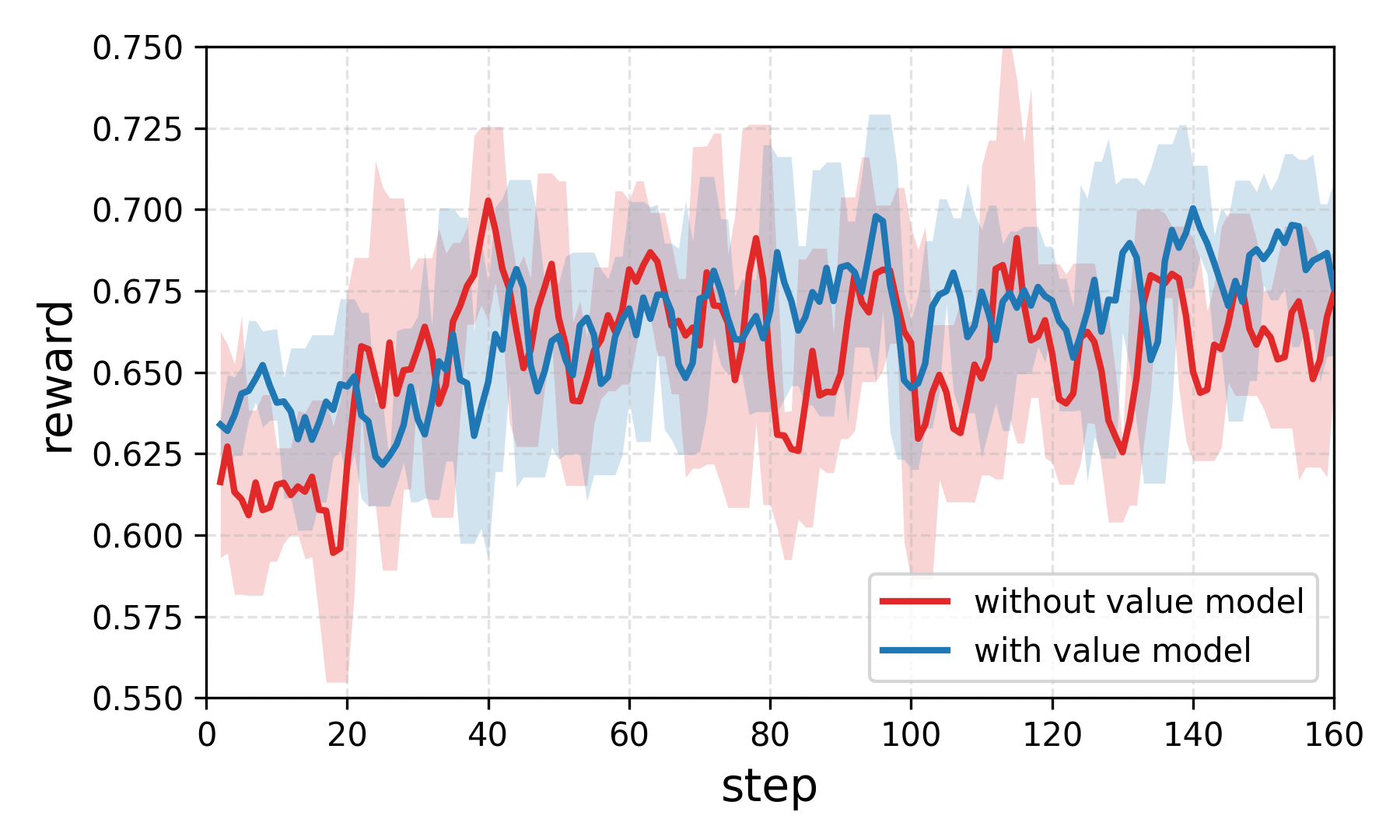}
  \end{subfigure}
  \vspace{-12pt}
  \caption{(a) RL training ablations for the full attention model Dream-7B-Coder-Instruct, focusing on coding tasks. (b) The comparison between using and not using a value model shows that incorporating a value model leads to fewer fluctuations during training. The experiments are all conducted on the 4B model with math tasks.}
  \label{fig:dyn_sta}
\end{figure}

\subsection{Also Fit for Full Attention Models and Coding Tasks}

To demonstrate the broad applicability of our method, we also experiment with a full attention model on coding RL tasks. We use Dream-7B-Coder-Instruct as the base model, fine-tuned on distillation data as a cold start before performing RL training (see details in Section~\ref{traindetailapp}). To accelerate training, we set the shrinkage parameter to $s = 8$. Similar to Table~\ref{tracesft}, we augment the training data for the two baseline methods (Section~\ref{traindetailapp}). We find that TraceRL converges faster and achieves the best performance compared with the baselines (see Figure~\ref{fig:dyn_sta} (a)). Notably, after training, we achieve 25.0\% accuracy on LiveCodeBench-V2, setting a new state of the art among open-source full-attention DLMs.

\subsection{Diffusion Value Model for Reducing Training Variance}

Training a policy model coupled with a value model can provide a variance-reducing baseline in LLM RL training \citep{schulman2017proximal, hu2025open}. We extend this idea to diffusion language models. In our diffusion-based value model, we assign the verifiable reward of the entire trajectory to the tokens in the final step, while leaving the reward for all preceding tokens as zero. We then use our proposed modeling to estimate values and advantages, jointly optimizing the policy and value models. This approach leads to reduced variance and fluctuations in the training process (see Figure~\ref{fig:dyn_sta}(b)) for the 4B model on math tasks.

\begin{table}[t]
\centering
\caption{Adapting block size $B = 4$ to $8$ on reasoning tasks through TraceRL. The reported values are accuracies of these baselines under dynamic sampling with threshold $0.9$.}
\label{tab:largerb}
\begin{tabular}{lcccc}
\toprule
\textbf{Model} & \textbf{inference block size} & \textbf{MATH500} & \textbf{GSM8K} & \textbf{LiveCodeBench} \\
\midrule
\multirow{2}{*}{Base model with $B = 4$} & 4 & 67.4 & 88.9 & 11.2 \\
                              & 8 & 60.2 & 83.0 & 9.8 \\
Enlarge B to 8 by TraceRL               & 8 & 67.7 & 88.7 & 10.8 \\
\bottomrule
\end{tabular}
\end{table}

\subsection{Scaling Block Size with TraceRL}

A block size of 4 can limit inference flexibility, and some inference acceleration methods depend on larger block sizes \citep{hong2025wide, song2025seed, wang2025diffusion}. Therefore, we explore the potential of using TraceRL to increase the block size. We first perform rollout with a block size of $B = 4$, then apply TraceRL with $B = 8$. After 60 steps, we switch the rollout to $B = 8$ and continue optimizing with $B = 8$ for 40 steps, resulting in a model that is familiar with and better adapted to larger block sizes. As shown in Table~\ref{tab:largerb}, this approach proves effective across different tasks, even though it is trained only on math tasks.

\subsection{Analysis on Acceleration Ratio and Response Length}

From Table~\ref{tab:acceffect}, we observe that TraceRL optimization accelerates dynamic sampling, achieving a $15.4\%$ speedup on MATH500. A straightforward explanation is that the model becomes more confident when encountering problems in the domain it has been optimized for, which in turn allows each step of dynamic sampling more likely to unmask more tokens under the same threshold. We also find that the average response length on complex math reasoning tasks increases, providing further evidence of improved reasoning ability.

\begin{table*}[t]
\centering
\caption{Acceleration ratio and Response Length Analysis. ``Acceleration" is defined as the ratio of response length to the total sampling steps for a task, and we report the average ratio across the tested dataset. We use dynamic sampling with threshold 0.9 here.}\label{tab:acceffect}
\begin{adjustbox}{max width=\textwidth} 
\begin{tabular}{l|ccc|ccc}
\toprule
\multirow{2}{*}{\textbf{Model}} &
  \multicolumn{3}{c|}{\textbf{MATH500}} &
  \multicolumn{3}{c}{\textbf{LiveBench}} \\
& accelerated & total step & avg len.  & accelerated & total step & avg len. \\
\midrule
SDAR-4B-Chat        & 2.28 & 240 & 548 tok. &  1.56 & 119 & 181 tok. \\
TraDo-4B-Instruct   & 2.63 & 229 & 595 tok. & 1.61  & 154 & 238 tok.  \\
SDAR-8B-Chat  & 2.42 & 228 & 557 tok.  & 1.56  & 161 & 256 tok.  \\
TraDo-8B-Instruct   & 2.50 & 240 & 625 tok. &   1.60 & 151 & 233 tok.  \\
\bottomrule
\end{tabular}
\end{adjustbox}
\end{table*}

\section{Open-source Framework}

    We implement an open-source framework for training and deploying large diffusion language models, released at \href{https://github.com/Gen-Verse/dLLM-RL}{https://github.com/Gen-Verse/dLLM-RL}. Its main features are summarized as follows.

\subsection{Support for Different Model Structures}

Diffusion language models exhibit diverse architectures (see Section~\ref{prelim}). Our framework is comprehensive and supports major classes of large diffusion models, including pretrained full-attention models (e.g., the LLaDA series \citep{nie2025large, zhu2025llada} and MMaDA \citep{yang2025mmada}), adapted full-attention models (e.g., the Dream series \citep{dream2025, dreamcoder2025} and DiffuCoder \citep{gong2025diffucoder}), and block diffusion models (e.g., SDAR \citep{JetAstra2025} and our TraDo series).

\subsection{Accelerated Inference}

To improve sampling efficiency while avoiding performance degradation, multiple KV-cache techniques have been proposed for full-attention diffusion language models \citep{hu2025accelerating, liu2025dllm, ma2025dkv, arriola2025block, yu2025dimple, wu2025fast}. We extend Fast-DLLM \citep{wu2025fast} by adding a tunable window-size further horizon for each forward pass (see Section~\ref{appeval}). For block diffusion models, we employ JetEngine \citep{JetAstra2025} to accelerate inference. These acceleration techniques are applied in both inference and RL sampling steps.

\subsection{Diverse Post-training Method Implementations}

For supervised fine-tuning, we support fully random masking and semi-AR methods for full-attention models, as well as efficient semi-AR fine-tuning for block-attention models. Our framework includes TraceRL, coupled RL \citep{gong2025diffucoder}, and random masking RL \citep{yang2025mmada}, with accelerated inference steps. For block diffusion models, we also provide the option to use a diffusion-based value model. All training methods support multi-node training.

\section{Related Work}

\subsection{Enhancing Reasoning Capabilities of Large Language Models}

Approaches have been proposed to boost the reasoning abilities of large language models (LLMs), spanning Chain-of-Thought (CoT) prompting \citep{wei2022chain,yang2024buffer}—which elicits step-by-step rationales—and self-improvement optimization algorithms \citep{zelikman2024star,hosseini2024v,yang2024supercorrect}, in which the model iteratively refines its own answers through verification, reflection, or voting. Incorporating Long-CoT methods, such as self-checking and self-reflection patterns, significantly enhances the reasoning capabilities of LLMs. Reinforcement learning (RL) has been found to be able to incentivize the reasoning ability of LLMs \citep{guo2025deepseek,k1.5,openr1,hu2025open,yang2025reasonflux} and to improve their ability to solve complex tasks \citep{jiang2025deepretrieval, jin2025search, yang2024acecode, wang2025co, zou2025reasonflux}.


Despite the notable reasoning power of Long-CoT models, their inference speed remains prohibitively slow, highlighting the necessity and potential for faster inference methods \citep{sui2025stop}. Diffusion language models, with their parallel decoding capabilities and bidirectional mechanisms, present a promising direction to overcome these limitations and achieve more efficient reasoning inference.

\subsection{RL for Diffusion Language Model}

Reinforcement learning methods for full-attention models have been studied in prior work \citep{yang2025mmada,gong2025diffucoder,zhao2025d1}. In MMaDA, \citet{yang2025mmada} introduce random masking for each rollout sample and optimize a PPO objective to perform RL training. \citet{gong2025diffucoder} propose augmenting each rollout $x_0$ with not only a single noise sample $x_t$ but also a complementary counterpart $\widehat{x}_t$, thereby doubling the effective training data. They provide a theoretical analysis showing that this strategy reduces variance, and empirical results confirm faster and more stable optimization compared to random-masking RL. Nevertheless, these approaches overlook information contained in the sampling trajectory.

\subsection{Diffusion Models for Language Modeling}

Extending the iterative-refinement paradigm of diffusion language models \citep{ho2020denoising,song2019generative} to discrete language data is not straightforward. One line of research projects tokens into a continuous latent space where standard Gaussian diffusion can operate \citep{dieleman2022continuous, graves2023bayesian, gulrajani2023likelihood}.
A parallel thread directly defines the forward noising process on the token simplex via state-transition matrices. Structured Diffusion LM \citep{austin2021structured} introduced a categorical forward kernel (e.g.\ uniform corruption or bit-flip), and trained a transformer to invert it. Subsequent work simplified or re-parameterised the transition matrix to reduce variance and improve stability \citep{sahoo2024simple,shi2024simplified,ou2024your}. 

Masked Diffusion Models (MDMs) have emerged as a scalable architecture for large-scale diffusion language models \citep{nie2025large,dream2025,yang2025mmada}. MDMs naturally leverage bidirectional attention, yielding stronger global consistency than left-to-right autoregressive transformers \citep{ye2024beyond,zhang2023planner}, while also supporting parallel decoding, which enables substantial inference acceleration \citep{labs2025mercury,google_deepmind_gemini_diffusion_2025}.

\section{Conclusion}

We present a new reinforcement learning method for diffusion language models with diverse architectures. Through extensive experiments and evaluations, we demonstrate the effectiveness of this method across different RL tasks, resulting in three state-of-the-art diffusion language models. We also highlight its benefits for accelerating inference and scaling block size, which suggest promising directions for further exploration.

In particular, integrating the accelerated inference strengths of diffusion models with their potential for strong reasoning ability represents an exciting research avenue. While current long-CoT LLMs achieve strong performance on complex tasks, they suffer from extremely long reasoning times. Such integration could unlock new opportunities for deploying complex reasoning tasks efficiently at scale.
The proposed diffusion value model can incorporate process rewards, providing stronger supervision than a single verifiable reward \citep{lightman2023let}. Further exploration of process reward for TraceRL optimization remains an important direction for future work.

To accelerate progress in this area, facilitate reproducible research, and support practical applications, we release a comprehensive framework for building, training, and deploying diffusion language models with diverse architectures. This framework integrates accelerated inference techniques and incorporates them into its reinforcement learning pipelines. It supports implementations of various reinforcement learning and supervised fine-tuning methods.

\clearpage

\bibliography{main}

\begin{thebibliography}{62}
\providecommand{\natexlab}[1]{#1}
\providecommand{\url}[1]{\texttt{#1}}
\expandafter\ifx\csname urlstyle\endcsname\relax
  \providecommand{\doi}[1]{doi: #1}\else
  \providecommand{\doi}{doi: \begingroup \urlstyle{rm}\Url}\fi

\bibitem[Arriola et~al.(2025)Arriola, Gokaslan, Chiu, Yang, Qi, Han, Sahoo, and Kuleshov]{arriola2025block}
M.~Arriola, A.~Gokaslan, J.~T. Chiu, Z.~Yang, Z.~Qi, J.~Han, S.~S. Sahoo, and V.~Kuleshov.
\newblock Block diffusion: Interpolating between autoregressive and diffusion language models.
\newblock \emph{arXiv preprint arXiv:2503.09573}, 2025.

\bibitem[Austin et~al.(2021)Austin, Johnson, Ho, Tarlow, and Van Den~Berg]{austin2021structured}
J.~Austin, D.~D. Johnson, J.~Ho, D.~Tarlow, and R.~Van Den~Berg.
\newblock Structured denoising diffusion models in discrete state-spaces.
\newblock \emph{Advances in neural information processing systems}, 34:\penalty0 17981--17993, 2021.

\bibitem[Chang et~al.(2022)Chang, Zhang, Jiang, Liu, and Freeman]{chang2022maskgit}
H.~Chang, H.~Zhang, L.~Jiang, C.~Liu, and W.~T. Freeman.
\newblock Maskgit: Masked generative image transformer.
\newblock In \emph{Proceedings of the IEEE/CVF conference on computer vision and pattern recognition}, pages 11315--11325, 2022.

\bibitem[Cheng et~al.(2025)Cheng, Bian, Liu, Jiang, Liu, Zhang, Wang, Guo, Chen, Qi*, and Zhou]{JetAstra2025}
S.~Cheng, Y.~Bian, D.~Liu, Y.~Jiang, Y.~Liu, L.~Zhang, W.~Wang, Q.~Guo, K.~Chen, B.~Qi*, and B.~Zhou.
\newblock Sdar: A synergistic diffusion–autoregression paradigm for scalable sequence generation, 2025.
\newblock URL \url{https://github.com/JetAstra/SDAR}.

\bibitem[Cobbe et~al.(2021)Cobbe, Kosaraju, Bavarian, Chen, Jun, Kaiser, Plappert, Tworek, Hilton, Nakano, et~al.]{cobbe2021training}
K.~Cobbe, V.~Kosaraju, M.~Bavarian, M.~Chen, H.~Jun, L.~Kaiser, M.~Plappert, J.~Tworek, J.~Hilton, R.~Nakano, et~al.
\newblock Training verifiers to solve math word problems.
\newblock \emph{arXiv preprint arXiv:2110.14168}, 2021.

\bibitem[Dieleman et~al.(2022)Dieleman, Sartran, Roshannai, Savinov, Ganin, Richemond, Doucet, Strudel, Dyer, Durkan, et~al.]{dieleman2022continuous}
S.~Dieleman, L.~Sartran, A.~Roshannai, N.~Savinov, Y.~Ganin, P.~H. Richemond, A.~Doucet, R.~Strudel, C.~Dyer, C.~Durkan, et~al.
\newblock Continuous diffusion for categorical data.
\newblock \emph{arXiv preprint arXiv:2211.15089}, 2022.

\bibitem[Dubey et~al.(2024)Dubey, Jauhri, Pandey, Kadian, Al-Dahle, Letman, Mathur, Schelten, Yang, Fan, et~al.]{dubey2024llama}
A.~Dubey, A.~Jauhri, A.~Pandey, A.~Kadian, A.~Al-Dahle, A.~Letman, A.~Mathur, A.~Schelten, A.~Yang, A.~Fan, et~al.
\newblock The llama 3 herd of models.
\newblock \emph{arXiv e-prints}, pages arXiv--2407, 2024.

\bibitem[Gong et~al.(2025)Gong, Zhang, Zheng, Gu, Jaitly, Kong, and Zhang]{gong2025diffucoder}
S.~Gong, R.~Zhang, H.~Zheng, J.~Gu, N.~Jaitly, L.~Kong, and Y.~Zhang.
\newblock Diffucoder: Understanding and improving masked diffusion models for code generation.
\newblock \emph{arXiv preprint arXiv:2506.20639}, 2025.

\bibitem[{Google DeepMind}(2025)]{google_deepmind_gemini_diffusion_2025}
{Google DeepMind}.
\newblock Gemini diffusion.
\newblock \url{https://blog.google/technology/google-deepmind/gemini-diffusion/}, 2025.
\newblock Accessed: 2024-07-24.

\bibitem[Graves et~al.(2023)Graves, Srivastava, Atkinson, and Gomez]{graves2023bayesian}
A.~Graves, R.~K. Srivastava, T.~Atkinson, and F.~Gomez.
\newblock Bayesian flow networks.
\newblock \emph{arXiv preprint arXiv:2308.07037}, 2023.

\bibitem[Gulrajani and Hashimoto(2023)]{gulrajani2023likelihood}
I.~Gulrajani and T.~B. Hashimoto.
\newblock Likelihood-based diffusion language models.
\newblock \emph{Advances in Neural Information Processing Systems}, 36:\penalty0 16693--16715, 2023.

\bibitem[Guo et~al.(2025)Guo, Yang, Zhang, Song, Zhang, Xu, Zhu, Ma, Wang, Bi, et~al.]{guo2025deepseek}
D.~Guo, D.~Yang, H.~Zhang, J.~Song, R.~Zhang, R.~Xu, Q.~Zhu, S.~Ma, P.~Wang, X.~Bi, et~al.
\newblock Deepseek-r1: Incentivizing reasoning capability in llms via reinforcement learning.
\newblock \emph{arXiv preprint arXiv:2501.12948}, 2025.

\bibitem[Hendrycks et~al.(2021)Hendrycks, Burns, Kadavath, Arora, Basart, Tang, Song, and Steinhardt]{hendrycks2021measuring}
D.~Hendrycks, C.~Burns, S.~Kadavath, A.~Arora, S.~Basart, E.~Tang, D.~Song, and J.~Steinhardt.
\newblock Measuring mathematical problem solving with the math dataset.
\newblock \emph{arXiv preprint arXiv:2103.03874}, 2021.

\bibitem[Ho et~al.(2020)Ho, Jain, and Abbeel]{ho2020denoising}
J.~Ho, A.~Jain, and P.~Abbeel.
\newblock Denoising diffusion probabilistic models.
\newblock \emph{Advances in neural information processing systems}, 33:\penalty0 6840--6851, 2020.

\bibitem[Hong et~al.(2025)Hong, Yu, Ye, Huang, Zheng, Zhang, Wang, and Yao]{hong2025wide}
F.~Hong, G.~Yu, Y.~Ye, H.~Huang, H.~Zheng, Y.~Zhang, Y.~Wang, and J.~Yao.
\newblock Wide-in, narrow-out: Revokable decoding for efficient and effective dllms.
\newblock \emph{arXiv preprint arXiv:2507.18578}, 2025.

\bibitem[Hosseini et~al.(2024)Hosseini, Yuan, Malkin, Courville, Sordoni, and Agarwal]{hosseini2024v}
A.~Hosseini, X.~Yuan, N.~Malkin, A.~Courville, A.~Sordoni, and R.~Agarwal.
\newblock V-star: Training verifiers for self-taught reasoners.
\newblock \emph{arXiv preprint arXiv:2402.06457}, 2024.

\bibitem[Hu et~al.(2025{\natexlab{a}})Hu, Zhang, Han, Jiang, Zhang, and Shum]{hu2025open}
J.~Hu, Y.~Zhang, Q.~Han, D.~Jiang, X.~Zhang, and H.-Y. Shum.
\newblock Open-reasoner-zero: An open source approach to scaling up reinforcement learning on the base model.
\newblock \emph{arXiv preprint arXiv:2503.24290}, 2025{\natexlab{a}}.

\bibitem[Hu et~al.(2025{\natexlab{b}})Hu, Meng, Akhauri, Abdelfattah, Seo, Zhang, and Gupta]{hu2025accelerating}
Z.~Hu, J.~Meng, Y.~Akhauri, M.~S. Abdelfattah, J.-s. Seo, Z.~Zhang, and U.~Gupta.
\newblock Accelerating diffusion language model inference via efficient kv caching and guided diffusion.
\newblock \emph{arXiv preprint arXiv:2505.21467}, 2025{\natexlab{b}}.

\bibitem[{Hugging Face}(2025)]{openr1}
{Hugging Face}.
\newblock Open r1: A fully open reproduction of deepseek-r1, January 2025.
\newblock URL \url{https://github.com/huggingface/open-r1}.

\bibitem[Jaghouar et~al.(2024)Jaghouar, Ong, Basra, Obeid, Straube, Keiblinger, Bakouch, Atkins, Panahi, Goddard, et~al.]{jaghouar2024intellect}
S.~Jaghouar, J.~M. Ong, M.~Basra, F.~Obeid, J.~Straube, M.~Keiblinger, E.~Bakouch, L.~Atkins, M.~Panahi, C.~Goddard, et~al.
\newblock Intellect-1 technical report.
\newblock \emph{arXiv preprint arXiv:2412.01152}, 2024.

\bibitem[Jain et~al.(2024)Jain, Han, Gu, Li, Yan, Zhang, Wang, Solar-Lezama, Sen, and Stoica]{jain2024livecodebench}
N.~Jain, K.~Han, A.~Gu, W.-D. Li, F.~Yan, T.~Zhang, S.~Wang, A.~Solar-Lezama, K.~Sen, and I.~Stoica.
\newblock Livecodebench: Holistic and contamination free evaluation of large language models for code.
\newblock \emph{arXiv preprint arXiv:2403.07974}, 2024.

\bibitem[Jiang et~al.(2025)Jiang, Lin, Cao, Tian, Kang, Wang, Sun, and Han]{jiang2025deepretrieval}
P.~Jiang, J.~Lin, L.~Cao, R.~Tian, S.~Kang, Z.~Wang, J.~Sun, and J.~Han.
\newblock Deepretrieval: Hacking real search engines and retrievers with large language models via reinforcement learning.
\newblock \emph{arXiv preprint arXiv:2503.00223}, 2025.

\bibitem[Jin et~al.(2025)Jin, Zeng, Yue, Yoon, Arik, Wang, Zamani, and Han]{jin2025search}
B.~Jin, H.~Zeng, Z.~Yue, J.~Yoon, S.~Arik, D.~Wang, H.~Zamani, and J.~Han.
\newblock Search-r1: Training llms to reason and leverage search engines with reinforcement learning.
\newblock \emph{arXiv preprint arXiv:2503.09516}, 2025.

\bibitem[Kim et~al.(2025)Kim, Shah, Kontonis, Kakade, and Chen]{kim2025train}
J.~Kim, K.~Shah, V.~Kontonis, S.~Kakade, and S.~Chen.
\newblock Train for the worst, plan for the best: Understanding token ordering in masked diffusions.
\newblock \emph{arXiv preprint arXiv:2502.06768}, 2025.

\bibitem[Labs et~al.(2025)Labs, Khanna, Kharbanda, Li, Varma, Wang, Birnbaum, Luo, Miraoui, Palrecha, et~al.]{labs2025mercury}
I.~Labs, S.~Khanna, S.~Kharbanda, S.~Li, H.~Varma, E.~Wang, S.~Birnbaum, Z.~Luo, Y.~Miraoui, A.~Palrecha, et~al.
\newblock Mercury: Ultra-fast language models based on diffusion.
\newblock \emph{arXiv preprint arXiv:2506.17298}, 2025.

\bibitem[Li et~al.(2022)Li, Choi, Chung, Kushman, Schrittwieser, Leblond, Eccles, Keeling, Gimeno, Lago, Hubert, Choy, de~Masson~d’Autume, Babuschkin, Chen, Huang, Welbl, Gowal, Cherepanov, Molloy, Mankowitz, Robson, Kohli, de~Freitas, Kavukcuoglu, and Vinyals]{li2022alphacode}
Y.~Li, D.~Choi, J.~Chung, N.~Kushman, J.~Schrittwieser, R.~Leblond, T.~Eccles, J.~Keeling, F.~Gimeno, A.~D. Lago, T.~Hubert, P.~Choy, C.~de~Masson~d’Autume, I.~Babuschkin, X.~Chen, P.-S. Huang, J.~Welbl, S.~Gowal, A.~Cherepanov, J.~Molloy, D.~J. Mankowitz, E.~S. Robson, P.~Kohli, N.~de~Freitas, K.~Kavukcuoglu, and O.~Vinyals.
\newblock Competition-level code generation with alphacode.
\newblock \emph{Science}, 378\penalty0 (6624):\penalty0 1092--1097, 2022.
\newblock \doi{10.1126/science.abq1158}.
\newblock URL \url{https://www.science.org/doi/10.1126/science.abq1158}.

\bibitem[Lightman et~al.(2023)Lightman, Kosaraju, Burda, Edwards, Baker, Lee, Leike, Schulman, Sutskever, and Cobbe]{lightman2023let}
H.~Lightman, V.~Kosaraju, Y.~Burda, H.~Edwards, B.~Baker, T.~Lee, J.~Leike, J.~Schulman, I.~Sutskever, and K.~Cobbe.
\newblock Let's verify step by step.
\newblock In \emph{The Twelfth International Conference on Learning Representations}, 2023.

\bibitem[Liu et~al.(2025)Liu, Yang, Zhang, Chen, Zou, Wei, Wang, and Zhang]{liu2025dllm}
Z.~Liu, Y.~Yang, Y.~Zhang, J.~Chen, C.~Zou, Q.~Wei, S.~Wang, and L.~Zhang.
\newblock dllm-cache: Accelerating diffusion large language models with adaptive caching.
\newblock \emph{arXiv preprint arXiv:2506.06295}, 2025.

\bibitem[Luo et~al.(2025)Luo, Tan, Huang, Patel, Ariyak, Wu, Shi, Xin, Cai, Weber, Zhang, Li, Popa, and Stoica]{deepcoder2025}
M.~Luo, S.~Tan, R.~Huang, A.~Patel, A.~Ariyak, Q.~Wu, X.~Shi, R.~Xin, C.~Cai, M.~Weber, C.~Zhang, L.~E. Li, R.~A. Popa, and I.~Stoica.
\newblock Deepcoder: A fully open-source 14b coder at o3-mini level.
\newblock \url{https://pretty-radio-b75.notion.site/DeepCoder-A-Fully-Open-Source-14B-Coder-at-O3-mini-Level-1cf81902c14680b3bee5eb349a512a51}, 2025.
\newblock Notion Blog.

\bibitem[Ma et~al.(2025)Ma, Yu, Fang, and Wang]{ma2025dkv}
X.~Ma, R.~Yu, G.~Fang, and X.~Wang.
\newblock dkv-cache: The cache for diffusion language models.
\newblock \emph{arXiv preprint arXiv:2505.15781}, 2025.

\bibitem[{Mathematical Association of America, American Mathematics Competitions}(2024)]{AIME2024}
{Mathematical Association of America, American Mathematics Competitions}.
\newblock American invitational mathematics examination (aime) 2024: Aime i and aime ii.
\newblock \url{https://artofproblemsolving.com/wiki/index.php/AIME_Problems_and_Solutions}, 2024.
\newblock Competition problems used as an evaluation dataset; original problems by MAA AMC.

\bibitem[Nie et~al.(2025)Nie, Zhu, You, Zhang, Ou, Hu, Zhou, Lin, Wen, and Li]{nie2025large}
S.~Nie, F.~Zhu, Z.~You, X.~Zhang, J.~Ou, J.~Hu, J.~Zhou, Y.~Lin, J.-R. Wen, and C.~Li.
\newblock Large language diffusion models.
\newblock \emph{arXiv preprint arXiv:2502.09992}, 2025.

\bibitem[Ou et~al.(2024)Ou, Nie, Xue, Zhu, Sun, Li, and Li]{ou2024your}
J.~Ou, S.~Nie, K.~Xue, F.~Zhu, J.~Sun, Z.~Li, and C.~Li.
\newblock Your absorbing discrete diffusion secretly models the conditional distributions of clean data.
\newblock \emph{arXiv preprint arXiv:2406.03736}, 2024.

\bibitem[Sahoo et~al.(2024)Sahoo, Arriola, Schiff, Gokaslan, Marroquin, Chiu, Rush, and Kuleshov]{sahoo2024simple}
S.~Sahoo, M.~Arriola, Y.~Schiff, A.~Gokaslan, E.~Marroquin, J.~Chiu, A.~Rush, and V.~Kuleshov.
\newblock Simple and effective masked diffusion language models.
\newblock \emph{Advances in Neural Information Processing Systems}, 37:\penalty0 130136--130184, 2024.

\bibitem[Schulman et~al.(2017)Schulman, Wolski, Dhariwal, Radford, and Klimov]{schulman2017proximal}
J.~Schulman, F.~Wolski, P.~Dhariwal, A.~Radford, and O.~Klimov.
\newblock Proximal policy optimization algorithms.
\newblock \emph{arXiv preprint arXiv:1707.06347}, 2017.

\bibitem[Shi et~al.(2024)Shi, Han, Wang, Doucet, and Titsias]{shi2024simplified}
J.~Shi, K.~Han, Z.~Wang, A.~Doucet, and M.~Titsias.
\newblock Simplified and generalized masked diffusion for discrete data.
\newblock \emph{Advances in neural information processing systems}, 37:\penalty0 103131--103167, 2024.

\bibitem[Song and Ermon(2019)]{song2019generative}
Y.~Song and S.~Ermon.
\newblock Generative modeling by estimating gradients of the data distribution.
\newblock \emph{Advances in neural information processing systems}, 32, 2019.

\bibitem[Song et~al.(2025)Song, Zhang, Luo, Gao, Xia, Luo, Li, Yang, Yu, Qu, et~al.]{song2025seed}
Y.~Song, Z.~Zhang, C.~Luo, P.~Gao, F.~Xia, H.~Luo, Z.~Li, Y.~Yang, H.~Yu, X.~Qu, et~al.
\newblock Seed diffusion: A large-scale diffusion language model with high-speed inference.
\newblock \emph{arXiv preprint arXiv:2508.02193}, 2025.

\bibitem[Sui et~al.(2025)Sui, Chuang, Wang, Zhang, Zhang, Yuan, Liu, Wen, Zhong, Chen, et~al.]{sui2025stop}
Y.~Sui, Y.-N. Chuang, G.~Wang, J.~Zhang, T.~Zhang, J.~Yuan, H.~Liu, A.~Wen, S.~Zhong, H.~Chen, et~al.
\newblock Stop overthinking: A survey on efficient reasoning for large language models.
\newblock \emph{arXiv preprint arXiv:2503.16419}, 2025.

\bibitem[Team et~al.(2025)Team, Du, Gao, Xing, Jiang, Chen, Li, Xiao, Du, Liao, et~al.]{k1.5}
K.~Team, A.~Du, B.~Gao, B.~Xing, C.~Jiang, C.~Chen, C.~Li, C.~Xiao, C.~Du, C.~Liao, et~al.
\newblock Kimi k1. 5: Scaling reinforcement learning with llms.
\newblock \emph{arXiv preprint arXiv:2501.12599}, 2025.

\bibitem[Wang et~al.(2025{\natexlab{a}})Wang, Xu, Jin, Jin, Zhang, and Deng]{wang2025diffusion}
X.~Wang, C.~Xu, Y.~Jin, J.~Jin, H.~Zhang, and Z.~Deng.
\newblock Diffusion llms can do faster-than-ar inference via discrete diffusion forcing.
\newblock \emph{arXiv preprint arXiv:2508.09192}, 2025{\natexlab{a}}.

\bibitem[Wang et~al.(2025{\natexlab{b}})Wang, Yang, Tian, Shen, and Wang]{wang2025co}
Y.~Wang, L.~Yang, Y.~Tian, K.~Shen, and M.~Wang.
\newblock Co-evolving llm coder and unit tester via reinforcement learning.
\newblock \emph{arXiv preprint arXiv:2506.03136}, 2025{\natexlab{b}}.

\bibitem[Wei et~al.(2022)Wei, Wang, Schuurmans, Bosma, Xia, Chi, Le, Zhou, et~al.]{wei2022chain}
J.~Wei, X.~Wang, D.~Schuurmans, M.~Bosma, F.~Xia, E.~Chi, Q.~V. Le, D.~Zhou, et~al.
\newblock Chain-of-thought prompting elicits reasoning in large language models.
\newblock \emph{Advances in neural information processing systems}, 35:\penalty0 24824--24837, 2022.

\bibitem[White et~al.(2024)White, Dooley, Roberts, Pal, Feuer, Jain, Shwartz-Ziv, Jain, Saifullah, Naidu, et~al.]{white2024livebench}
C.~White, S.~Dooley, M.~Roberts, A.~Pal, B.~Feuer, S.~Jain, R.~Shwartz-Ziv, N.~Jain, K.~Saifullah, S.~Naidu, et~al.
\newblock Livebench: A challenging, contamination-free llm benchmark.
\newblock \emph{arXiv preprint arXiv:2406.19314}, 2024.

\bibitem[Wu et~al.(2025)Wu, Zhang, Xue, Liu, Diao, Zhu, Luo, Han, and Xie]{wu2025fast}
C.~Wu, H.~Zhang, S.~Xue, Z.~Liu, S.~Diao, L.~Zhu, P.~Luo, S.~Han, and E.~Xie.
\newblock Fast-dllm: Training-free acceleration of diffusion llm by enabling kv cache and parallel decoding.
\newblock \emph{arXiv preprint arXiv:2505.22618}, 2025.

\bibitem[Xie et~al.(2025)Xie, Ye, Zheng, Gao, Dong, Wu, Zhao, Gong, Jiang, Li, and Kong]{dreamcoder2025}
Z.~Xie, J.~Ye, L.~Zheng, J.~Gao, J.~Dong, Z.~Wu, X.~Zhao, S.~Gong, X.~Jiang, Z.~Li, and L.~Kong.
\newblock Dream-coder 7b, 2025.
\newblock URL \url{https://hkunlp.github.io/blog/2025/dream-coder}.

\bibitem[Yang et~al.(2024{\natexlab{a}})Yang, Yang, Zhang, Hui, Zheng, Yu, Li, Liu, Huang, Wei, et~al.]{yang2024qwen2}
A.~Yang, B.~Yang, B.~Zhang, B.~Hui, B.~Zheng, B.~Yu, C.~Li, D.~Liu, F.~Huang, H.~Wei, et~al.
\newblock Qwen2. 5 technical report.
\newblock \emph{arXiv preprint arXiv:2412.15115}, 2024{\natexlab{a}}.

\bibitem[Yang et~al.(2024{\natexlab{b}})Yang, Kang, Shi, and Lo]{yang2024acecode}
C.~Yang, H.~J. Kang, J.~Shi, and D.~Lo.
\newblock Acecode: A reinforcement learning framework for aligning code efficiency and correctness in code language models.
\newblock \emph{arXiv preprint arXiv:2412.17264}, 2024{\natexlab{b}}.

\bibitem[Yang et~al.(2024{\natexlab{c}})Yang, Yu, Zhang, Cao, Xu, Zhang, Gonzalez, and Cui]{yang2024buffer}
L.~Yang, Z.~Yu, T.~Zhang, S.~Cao, M.~Xu, W.~Zhang, J.~E. Gonzalez, and B.~Cui.
\newblock Buffer of thoughts: Thought-augmented reasoning with large language models.
\newblock \emph{Advances in Neural Information Processing Systems}, 37:\penalty0 113519--113544, 2024{\natexlab{c}}.

\bibitem[Yang et~al.(2024{\natexlab{d}})Yang, Yu, Zhang, Xu, Gonzalez, Cui, and Yan]{yang2024supercorrect}
L.~Yang, Z.~Yu, T.~Zhang, M.~Xu, J.~E. Gonzalez, B.~Cui, and S.~Yan.
\newblock Supercorrect: Advancing small llm reasoning with thought template distillation and self-correction.
\newblock \emph{arXiv preprint arXiv:2410.09008}, 2024{\natexlab{d}}.

\bibitem[Yang et~al.(2025{\natexlab{a}})Yang, Tian, Li, Zhang, Shen, Tong, and Wang]{yang2025mmada}
L.~Yang, Y.~Tian, B.~Li, X.~Zhang, K.~Shen, Y.~Tong, and M.~Wang.
\newblock Mmada: Multimodal large diffusion language models.
\newblock \emph{arXiv preprint arXiv:2505.15809}, 2025{\natexlab{a}}.

\bibitem[Yang et~al.(2025{\natexlab{b}})Yang, Yu, Cui, and Wang]{yang2025reasonflux}
L.~Yang, Z.~Yu, B.~Cui, and M.~Wang.
\newblock Reasonflux: Hierarchical llm reasoning via scaling thought templates.
\newblock \emph{arXiv preprint arXiv:2502.06772}, 2025{\natexlab{b}}.

\bibitem[Ye et~al.(2024)Ye, Gao, Gong, Zheng, Jiang, Li, and Kong]{ye2024beyond}
J.~Ye, J.~Gao, S.~Gong, L.~Zheng, X.~Jiang, Z.~Li, and L.~Kong.
\newblock Beyond autoregression: Discrete diffusion for complex reasoning and planning.
\newblock \emph{arXiv preprint arXiv:2410.14157}, 2024.

\bibitem[Ye et~al.(2025)Ye, Xie, Zheng, Gao, Wu, Jiang, Li, and Kong]{dream2025}
J.~Ye, Z.~Xie, L.~Zheng, J.~Gao, Z.~Wu, X.~Jiang, Z.~Li, and L.~Kong.
\newblock Dream 7b, 2025.
\newblock URL \url{https://hkunlp.github.io/blog/2025/dream}.

\bibitem[Yu et~al.(2025{\natexlab{a}})Yu, Zhang, Zhu, Yuan, Zuo, Yue, Dai, Fan, Liu, Liu, et~al.]{yu2025dapo}
Q.~Yu, Z.~Zhang, R.~Zhu, Y.~Yuan, X.~Zuo, Y.~Yue, W.~Dai, T.~Fan, G.~Liu, L.~Liu, et~al.
\newblock Dapo: An open-source llm reinforcement learning system at scale.
\newblock \emph{arXiv preprint arXiv:2503.14476}, 2025{\natexlab{a}}.

\bibitem[Yu et~al.(2025{\natexlab{b}})Yu, Ma, and Wang]{yu2025dimple}
R.~Yu, X.~Ma, and X.~Wang.
\newblock Dimple: Discrete diffusion multimodal large language model with parallel decoding.
\newblock \emph{arXiv preprint arXiv:2505.16990}, 2025{\natexlab{b}}.

\bibitem[Zelikman et~al.(2024)Zelikman, Wu, Mu, and Goodman]{zelikman2024star}
E.~Zelikman, Y.~Wu, J.~Mu, and N.~D. Goodman.
\newblock Star: Self-taught reasoner bootstrapping reasoning with reasoning.
\newblock In \emph{Proc. the 36th International Conference on Neural Information Processing Systems}, volume 1126, 2024.

\bibitem[Zhang et~al.(2023)Zhang, Gu, Wu, Zhai, Susskind, and Jaitly]{zhang2023planner}
Y.~Zhang, J.~Gu, Z.~Wu, S.~Zhai, J.~Susskind, and N.~Jaitly.
\newblock Planner: Generating diversified paragraph via latent language diffusion model.
\newblock \emph{Advances in Neural Information Processing Systems}, 36:\penalty0 80178--80190, 2023.

\bibitem[Zhao et~al.(2025)Zhao, Gupta, Zheng, and Grover]{zhao2025d1}
S.~Zhao, D.~Gupta, Q.~Zheng, and A.~Grover.
\newblock d1: Scaling reasoning in diffusion large language models via reinforcement learning.
\newblock \emph{arXiv preprint arXiv:2504.12216}, 2025.

\bibitem[Zheng et~al.(2024)Zheng, Chen, Mao, Liu, Zhu, and Zhang]{zheng2024masked}
K.~Zheng, Y.~Chen, H.~Mao, M.-Y. Liu, J.~Zhu, and Q.~Zhang.
\newblock Masked diffusion models are secretly time-agnostic masked models and exploit inaccurate categorical sampling.
\newblock \emph{arXiv preprint arXiv:2409.02908}, 2024.

\bibitem[Zhu et~al.(2025)Zhu, Wang, Nie, Zhang, Wu, Hu, Zhou, Chen, Lin, Wen, et~al.]{zhu2025llada}
F.~Zhu, R.~Wang, S.~Nie, X.~Zhang, C.~Wu, J.~Hu, J.~Zhou, J.~Chen, Y.~Lin, J.-R. Wen, et~al.
\newblock Llada 1.5: Variance-reduced preference optimization for large language diffusion models.
\newblock \emph{arXiv preprint arXiv:2505.19223}, 2025.

\bibitem[Zou et~al.(2025)Zou, Yang, Gu, Qiu, Shen, He, and Wang]{zou2025reasonflux}
J.~Zou, L.~Yang, J.~Gu, J.~Qiu, K.~Shen, J.~He, and M.~Wang.
\newblock Reasonflux-prm: Trajectory-aware prms for long chain-of-thought reasoning in llms.
\newblock \emph{arXiv preprint arXiv:2506.18896}, 2025.

\end{thebibliography}

\clearpage

\appendix

\section{Theoretical Results}
\label{proofsect}

\begin{proposition}[Token-wise return and advantage from step-wise recursions]
\label{prop:token-from-trace}
Let a trajectory $\tau$ be partitioned into trace steps $\tau(1),\dots,\tau(|\tau|)$, and let $t_j$ denote the unique step index with $j\in\tau(t_j)$.
For token-wise rewards $r_j$ and token-wise values $V^{\text{old}}_j$, define the step-wise aggregates
\[
r_t^{\star}\;\coloneqq\; \frac{1}{|\tau(t)|}\sum_{l\in\tau(t)} r_l,
\qquad
V_{t}^{\star,\text{old}}\;\coloneqq\; \frac{1}{|\tau(t)|}\sum_{l\in\tau(t)} V^{\text{old}}_l.
\]
Let the step-wise return and GAE be given by
\[
R_t^{\star} \;=\; r_t^{\star} + \gamma R_{t+1}^{\star},\quad R_{|\tau|+1}^{\star}=0,\qquad
\delta_t^{\star} \;=\; r_t^{\star} - V_{t}^{\star,\text{old}} + \gamma V_{t+1}^{\star,\text{old}},
\]
\[
A_t^{\star}\;=\;\sum_{k=0}^{|\tau|-t} (\gamma\lambda)^k \delta_{t+k}^{\star},\quad A_{|\tau|+1}^{\star}=0,\quad V_{|\tau|+1}^{\star,\text{old}}=0.
\]
Define token-wise quantities
\[
R_j \;\coloneqq\; r_j + \gamma R_{t_j+1}^{\star},
\qquad
A_j \;\coloneqq\; r_j - V^{\text{old}}_j + \gamma V_{t_j+1}^{\star,\text{old}} + \gamma\lambda A_{t_j+1}^{\star}.
\]
Then the following explicit expressions hold:
\begin{align}
R_j
&= r_j \;+\; \sum_{k=1}^{|\tau|-t_j} \gamma^k\,\frac{1}{|\tau(t_j + k)|}\sum_{l\in\tau(t_j+k)} r_l,
\tag{1}\label{eq:Rj-explicit}
\\
A_j
&= r_j - V^{\text{old}}_j \;+\; \sum_{k=1}^{|\tau|-t_j} (\gamma\lambda)^k\,\frac{1}{|\tau(t_j + k)|}
\sum_{l\in\tau(t_j+k)}\!\Big(r_l + \tfrac{1-\lambda}{\lambda}\,V^{\text{old}}_l\Big).
\tag{2}\label{eq:Aj-explicit}
\end{align}
For the boundary case $\lambda=0$, one has directly $A_j = r_j - V^{\text{old}}_j + \gamma V_{t_j+1}^{\star,\text{old}}$.
\end{proposition}

\begin{proof}
\textbf{Return.}
Unroll the recursion $R_{t_j+1}^{\star}=\sum_{k=0}^{|\tau|-t_j-1}\gamma^k r_{t_j+1+k}^{\star}$ and substitute into
$R_j=r_j+\gamma R_{t_j+1}^{\star}$ to obtain
$R_j=r_j+\sum_{k=1}^{|\tau|-t_j}\gamma^k r_{t_j+k}^{\star}$.
Using $r_{t}^{\star}=\frac{1}{|\tau(t)|}\sum_{l\in\tau(t)} r_l$ gives \eqref{eq:Rj-explicit}.

\textbf{Advantage.}
Write
$A_{t_j+1}^{\star}=\sum_{k=0}^{|\tau|-t_j-1} (\gamma\lambda)^k \delta_{t_j+1+k}^{\star}$
and substitute $\delta_{t}^{\star}=r_t^{\star}-V_t^{\star,\text{old}}+\gamma V_{t+1}^{\star,\text{old}}$.
Reindex to start at $k=1$:
\[
A_j
= (r_j - V^{\text{old}}_j)
+ \sum_{k=1}^{|\tau|-t_j} (\gamma\lambda)^k r_{t_j+k}^{\star}
+ \gamma(1-\lambda)\sum_{k=1}^{|\tau|-t_j} (\gamma\lambda)^{k-1} V_{t_j+k}^{\star,\text{old}}.
\]
The last line follows from collecting the $V^{\star,\text{old}}$ terms into a telescoping series whose coefficient is
$\gamma(1-\lambda)(\gamma\lambda)^{k-1}$ for step $t_j+k$ (including $k=1$).
Now use
$r_{t}^{\star}=\frac{1}{|\tau(t)|}\sum_{l\in\tau(t)} r_l$ and
$V_{t}^{\star,\text{old}}=\frac{1}{|\tau(t)|}\sum_{l\in\tau(t)} V^{\text{old}}_l$,
and note that
$\gamma(1-\lambda)(\gamma\lambda)^{k-1}=(\gamma\lambda)^k\cdot\frac{1-\lambda}{\lambda}$ for $\lambda>0$.
This yields \eqref{eq:Aj-explicit}.
For $\lambda=0$, the series vanish and the definition gives the stated one-step TD form.
\end{proof}

\begin{remark}[Special cases]
\textbf{(i) $(\gamma,\lambda)=(1,1)$.} One has
\[
R_j = r_j + \sum_{k=1}^{|\tau|-t_j} \frac{1}{|\tau(t_j+k)|}\sum_{l\in\tau(t_j+k)} r_l,
\qquad
A_j = R_j - V^{\text{old}}_j,
\]
i.e., undiscounted Monte Carlo return and advantage with a token-wise baseline.

\textbf{(ii) $(\gamma,\lambda)=(1,0)$.} The return remains the undiscounted form with $\gamma=1$ as above, and
\[
A_j = r_j - V^{\text{old}}_j + V_{t_j+1}^{\star,\text{old}},
\]
which is the one-step TD(0) advantage using the trace-level baseline $V_{t_j+1}^{\star,\text{old}}$.
\end{remark}

\section{Experimental Details}
\label{expdetails}

\subsection{Supervised Finetuning Methods Explorations}
\label{misdetail}

In the demonstration of Section~\ref{mismatch}, the training data consist of CoT responses generated by the Qwen2.5-32B-Instruct model \citep{yang2024qwen2} on 2,000 randomly selected tasks from the OpenR1-MATH training set \citep{openr1}, filtered to exclude problems solvable by the Qwen2.5-7B-Instruct model. We then collect the traces of these data using the model under evaluation. For each instance, the model is used to iteratively select the two tokens with the highest confidence, conditioned on all previously processed tokens. This procedure produces a trace, which is then aggregated by grouping every $l / 2$ neighboring tokens to obtain the final trace, where each step has length $l$.
We train on the collected data for one epoch using 64 A100 GPUs, with a learning rate of $1 \times 10^{-6}$ for Dream and $1 \times 10^{-5}$ for SDAR.

\subsection{Prompt Templates}
\label{apptemp}

\begin{tcolorbox}[colback=blue!5!white, colframe=gray!50!black, title=Math Prompt Templates, breakable, enhanced jigsaw]
\textbf{Dream}

\begin{lstlisting}[style=prompt]
'''<|im_start|>system\nYou are a helpful assistant.<|im_end|>\n<|im_start|>user\nYou need to put your final answer in \\boxed{}. This is the problem:\n{{problem}}<|im_end|>\n<|im_start|>assistant\n'''
\end{lstlisting}

\textbf{LLaDA}

\begin{lstlisting}[style=prompt]
"""<|startoftext|><|start_header_id|>user<|end_header_id|>You need to put your final answer in \\boxed{}. This is the problem:\n{{problem}}<|eot_id|><|startoftext|><|start_header_id|>assistant<|end_header_id|>\n"""
\end{lstlisting}

\textbf{TraDo (non-thinking)}

\begin{lstlisting}[style=prompt]
'''<|im_start|>user\n{{problem}}\nPlease reason step by step, and put your final answer within \\boxed{}.<|im_end|>\n<|im_start|>assistant\n'''
\end{lstlisting}

\textbf{TraDo (thinking)}

\begin{lstlisting}[style=prompt]
'''<|im_start|>user\nYou need to put your final answer in \\boxed{}. This is the problem:\n{{problem}}<|im_end|>\n<|im_start|>assistant<think>\n'''
\end{lstlisting}

\textbf{TraDo (non-CoT)}

\begin{lstlisting}[style=prompt]
'''<|im_start|>user\nYou need to put your final answer in \\boxed{}. This is the problem:\n{{problem}}<|im_end|>\n<|im_start|>assistant\n'''
\end{lstlisting}
\end{tcolorbox}

\newpage
\begin{tcolorbox}[colback=blue!5!white, colframe=gray!50!black, title=Code Prompt Templates, breakable, enhanced jigsaw]
\textbf{Dream}

\begin{lstlisting}[style=prompt]
'''<|im_start|>system\nYou are a helpful assistant.<|im_end|>\n<|im_start|>user\nThis is the problem:\n{{problem}}\nYou should put your code in ```python ```. Use input() to read input and print() to produce output in your script. <|im_end|>\n<|im_start|>assistant\n'''
\end{lstlisting}

\textbf{LLaDA}

\begin{lstlisting}[style=prompt]
'''<|startoftext|><|start_header_id|>user<|end_header_id|>This is the problem:\n{{problem}}\n You should put your code in ```python ```. Use input() to read input and print() to produce output in your script. <|eot_id|><|startoftext|><|start_header_id|>assistant<|end_header_id|>\n'''
\end{lstlisting}

\textbf{TraDo (non-thinking)}

\begin{lstlisting}[style=prompt]
'''<|im_start|>user\nThis is the problem:\n{{problem}}\nYou should put your code in ```python ```. Use input() to read input and print() to produce output in your script. <|im_end|>\n<|im_start|>assistant\n'''
\end{lstlisting}

\textbf{TraDo (thinking)}

\begin{lstlisting}[style=prompt]
'''<|im_start|>user\nThis is the problem:\n{{problem}}\nYou should put your code in ```python ```. Use input() to read input and print() to produce output in your script. <|im_end|>\n<|im_start|>assistant<think>\n'''
\end{lstlisting}
\end{tcolorbox}

We use the same prompt template for both reinforcement learning and evaluation. TraDo and SDAR models share the same prompt templates. For long-CoT mode, we adopt the ``thinking'' prompt. The non-CoT prompt is only used for SDAR in our toy experiment in Section~\ref{mismatch}, which evaluates the effectiveness of different SFT methods in terms of how efficiently they teach the model to use CoT. The prompts are chosen to align as closely as possible with the model’s preferences and to reduce formatting errors (e.g., failure to extract the final answer).

\subsection{Evaluation Details}
\label{appeval}

For the LLaDA model, we use a block size of 32, a response limit of 1024 (no output truncation observed), a temperature of 0.1, and a further horizon size of 128. The further horizon size is defined as the additional number of tokens forwarded beyond the target block, excluding the tokens within the target block itself. We find that using this extra window reduces the performance drop caused by employing the KV-cache. When applying the dynamic sampling method, we set the unmasking threshold to $\mathcal{T} = 0.95$.

For the Dream model, we also use a temperature of 0.1, and a further horizon size of 128. For math problems, the response limit is set to 1600, and for coding problems, it is set to 1024. When using dynamic sampling, we use a block size of 4 with threshold $\mathcal{T} = 0.95$. Under static sampling, we use a block size of 32.

For the SDAR and TraDo instruction models, we keep the pretrained block size of 4, a response limit of 2000, and a temperature of 1.0. With dynamic sampling, we use a threshold of $\mathcal{T} = 0.9$ and $top\text{-}k = 0$ (i.e., all tokens are kept). For static sampling, we set $top\text{-}k = 1$, following \citep{JetAstra2025}.
For the long-CoT model TraDo-8B-Thinking, we set the response limit to 30,000 during evaluation, and only use dynamic sampling for evaluation.

For AIME2024, we evaluated each problem 20 times. For all other test datasets, we evaluated 3 times and report the average accuracy. 

We use the KV-cache in all evaluations to accelerate inference. For full attention mask models, we adapt and improve the fast-dllm framework \citep{wu2025fast}. For block diffusion models, we use jetengine \citep{JetAstra2025} to achieve acceleration.

\subsection{RL Sampling Details}
\label{appsample}

We describe here the parameters used in our experiments. For the Dream model, we set the block size to 32, the temperature to 0.8, the further horizon size to 128, and the response limit to 1024. During each step, we sample 56 problems, with 8 responses generated for each problem. We employ static decoding (one token per step) to enhance sampling quality \citep{gong2025diffucoder}, using the KV-cache. For the SDAR models, we use the default block size of 4, dynamic decoding with threshold $\mathcal{T} = 0.9$, $top\text{-}k = 0$, temperature 1.0, and $top\text{-}p = 1.0$ (also applied during evaluation). During each step, we sample 128 problems, with 32 responses generated for each problem.

\subsection{Training Details}
\label{traindetailapp}

For full attention models, we use a learning rate of $1 \times 10^{-6}$ for both semi-autoregressive and fully random masking fine-tuning methods. For adapted models such as Dream, autoregressive (AR) training can be applied with a learning rate of $1 \times 10^{-5}$ as a cold start. Since Dream shares the same architecture as an AR model, this provides a quick and effective way to teach the model specific language patterns. For block attention models, we use a learning rate of $1 \times 10^{-6}$. The masking probability is uniformly sampled between 0.1 and 0.9.

During RL training, we set the learning rate to $1 \times 10^{-6}$, with $\beta = 0.01$ and $\epsilon = 0.2$. When using a value model, we find no significant difference between the common parameter choices $(\gamma, \lambda) = (1, 1)$ and $(\gamma, \lambda) = (1, 0)$ \citep{hu2025open}. By default, we use the $k=3$ estimator for KL. For math tasks, we use binary outcomes as verifiable rewards and retain only those tasks with accuracy between 0.2 and 0.8 for training \citep{yu2025dapo}. For coding tasks, we use as the reward the proportion of unit tests passed by the generated solutions. To accelerate training, we use 64 A100 GPUs for our demonstration and ablation experiments. However, all experiments can be conducted on 8 A100 GPUs.

In Figure~\ref{fig:dyn_sta} (a), for Trace RL we use a shrinkage parameter of $s = 8$, together with an average response length of 380 during the RL process, resulting in approximately $380 / 8 = 47.5 < 50$ forward passes per data point. We augment the training data for fully random masking and coupled methods by applying 25 independent random masks, yielding 50 training samples per data point for the coupled method and 25 training samples per data point for the random masking method. Following \citet{gong2025diffucoder}, we keep the number of training samples for random masking at half that of the coupled method. Before RL training in Figure~\ref{fig:dyn_sta}(a), we collect 1.7k random SFT training samples from CodeContest \citep{li2022alphacode}, with solutions generated by the Qwen2.5-32B-Instruct model. We find that including the eos token in the SFT data is beneficial for stabilizing RL training.

One noteworthy point is the number of pad tokens that need to be trained for each data point, which we denote as $n_{pad}$. Setting a large $n_{pad}$ leads to excessively large logits for the pad token and can cause the model to terminate inference prematurely. In our RL training, we set $n_{pad} = 0$, which works well. However, in our long-CoT SFT step, we find that setting $n_{pad} = 0$ can potentially cause the model to never stop generating output during inference, although one remedy is to add an eos token as the stop token. Therefore, choosing an appropriate $n_{pad}$ is vital for stable training.

\subsection{TraceRL Algorithm Pipeline}
\label{algapp}

\begin{tcolorbox}[enhanced,breakable,
  colback=white, colframe=black, boxrule=0.4pt, arc=2pt,
  left=6pt,right=6pt,top=6pt,bottom=6pt]
  \captionof{algorithm}{TraceRL (Trajectory-Aware RL for DLMs)}\label{alg:tracerl}
\begin{algorithmic}[1]
\STATE \textbf{Input:}
\STATE \quad 1) Task set $\mathcal{D}_{\text{task}} = \{Q_1,\dots,Q_N\}$.
\STATE \quad 2) Policy $\pi_{\theta_p}$ parameterized by $\theta_p$.
\STATE \quad 3) (Optional) Value network $V_{\theta_v}$ parameterized by $\theta_v$; flag \texttt{UseValue} $\in\{\text{True},\text{False}\}$.
\STATE \quad 4) Iterations $M$, rollouts per task $G$, PPO clip $\epsilon$, KL coefficient $\beta$.
\STATE \quad 5) Discount $\gamma$, GAE parameter $\lambda$, learning rates $(\eta_p,\eta_v)$.
\STATE \quad 6) Shrinkage parameter $s$ (aggregate every $s$ neighboring trace steps).
\STATE \quad 7) (Optional) Value update interval $E_v$ (update value every $E_v$ iterations).

\STATE \textbf{Initialize:} $\theta_p$ (and $\theta_v$ if \texttt{UseValue}).

\FOR{$t = 1$ to $M$ \textbf{or} not converged}
  \STATE $\pi_{\text{old}} \leftarrow \pi_{\theta_p}$ \quad \textit{// freeze old policy for PPO ratios \& KL}
  \IF{\texttt{UseValue}}
    \STATE $V_{\text{old}} \leftarrow$ freeze copy of $V_{\theta_v}$ \quad \textit{// stop-gradient baseline}
  \ENDIF

  \STATE \textbf{Sample rollouts (trajectory traces):}
  \FOR{\textbf{each} minibatch of tasks $Q \sim \mathcal{D}_{\text{task}}$}
    \FOR{\textbf{repeat} $G$ times}
      \STATE Generate a response by policy decoding to obtain a trace 
      $\tau = (\tau(1), \ldots, \tau(|\tau|))$, where $\tau(t)$ is the token set decoded at step $t$.
      \STATE Obtain verifiable reward $r$ (or optionally process-level token/step rewards).
      \STATE \textbf{Shrink the trace:} $\tau^s \leftarrow \text{shrink}(\tau, s)$, i.e., group every $s$ neighboring steps.
      \IF{\texttt{UseValue}}
  \STATE \textbf{Value inference:} $V^{\mathrm{old}}_j \leftarrow (V_{\mathrm{old}}(\tau))_j$ for $j\in\tau$;\quad
  $V_{t}^{\star,\mathrm{old}} \leftarrow \frac{1}{|\tau(t)|}\sum_{j\in\tau(t)} V^{\mathrm{old}}_j$.
  \STATE \textbf{Build returns/advantages on the trace:}
  \STATE \quad Calculate step-wise reward: $r_t^{\star} $, return $R^{\star}_{t}$, advantage $A^{\star}_{t}$.
  \STATE \quad Map to tokens ($j\in\tau(t)$): $R_j \leftarrow r_j + \gamma R_{t+1}^{\star}$,\;\;
  $A_j \leftarrow (r_j - V^{\mathrm{old}}_j) + \gamma V_{t+1}^{\star,\mathrm{old}} + \gamma\lambda A_{t+1}^{\star}$.
  \STATE \quad (Optional) normalize advantages: $A_j \leftarrow \mathrm{normalize}(A_j)$.
\ENDIF
      \STATE Store $(Q, \tau^s, \text{advantages}, \pi_{\text{old}})$ for policy update.
    \ENDFOR
  \ENDFOR

  \STATE \textbf{Build grouped dataset} $\mathcal{D}_{\mathrm{grp}} \!\leftarrow\! \{(\mathbf{p}_g,\mathcal{O}_g,A_g,\pi_{\mathrm{old}})\}$.
\FOR{$e = 1$ to $K$}
  \STATE Sample minibatch $\mathcal{M} \subset \mathcal{D}_{\mathrm{grp}}$; compute ratios $r_j \!=\! \frac{\pi_{\theta_p}(o_j\mid \mathbf{p}_g)}{\pi_{\mathrm{old}}(o_j\mid \mathbf{p}_g)}$.
  \STATE \textbf{Policy step:} maximize $\mathcal{J}_{\mathrm{policy}}(\theta_p)$ (Eq.~\eqref{objpolicy}); \quad
  $\theta_p \leftarrow \theta_p + \eta_p \nabla_{\theta_p}\mathcal{J}_{\mathrm{policy}}$.
  \IF{\texttt{UseValue} \textbf{and} $(t \bmod E_v = 0)$}
    \STATE \textbf{Value step:} minimize $\mathcal{J}_{\mathrm{value}}(\theta_v)$ (Eq.~\eqref{objvalue}); \quad
    $\theta_v \leftarrow \theta_v - \eta_v \nabla_{\theta_v}\mathcal{J}_{\mathrm{value}}$.
  \ENDIF
\ENDFOR

\ENDFOR

\STATE \textbf{Output:} Trained policy $\pi_{\theta_p}$ (and value $V_{\theta_v}$ if used).
\end{algorithmic}
\end{tcolorbox}

\section{Additional Experimental Details}

\subsection{Average Response Length of TraDo-8B-Thinking}

\begin{table}[h]
\centering
\caption{Average response length of TraDo-8B-Thinking on different benchmarks.}
\label{tab:avglen}
\begin{tabular}{lccccc}
\toprule
 & \textbf{MATH500} & \textbf{AIME2024} & \textbf{GSM8K} & \textbf{LiveCodeBench} & \textbf{LiveBench} \\
\midrule
\textbf{Avg. length} & 5872 & 19397 & 1458 & 16291 & 16233 \\
\bottomrule
\end{tabular}
\end{table}

\subsection{Example of Long-CoT outputs}

\begin{tcolorbox}[colback=blue!5!white, colframe=gray!50!black, title=TraDo-8B-Thinking Example Output, breakable] \textbf{task:} What is the smallest positive perfect cube that can be written as the sum of three consecutive integers? \textbf{response:} Okay, so I need to find the smallest positive perfect cube that can be written as the sum of three consecutive integers. Let me start by understanding the problem step by step.

First, a perfect cube is a number that can be expressed as n³ where n is a positive integer. For example, 1, 8, 27, 64, etc. The question is asking for the smallest one of these that can also be written as the sum of three consecutive integers. 

Now, three consecutive integers can be represented as x, x+1, and x+2, right? So their sum would be x + (x+1) + (x+2). Let me compute that first. 

Adding those together: x + x + 1 + x + 2 = 3x + 3. So the sum of three consecutive integers is 3x + 3, which can be factored as 3(x + 1). 

Hmm, so the sum is always a multiple of 3. Therefore, any perfect cube that can be written as such a sum must be divisible by 3? Because it's three times some integer. 

So, the cube must be divisible by 3. Let me note that down. 

Therefore, the smallest perfect cube divisible by 3 is 27 (since 3³ = 27), but I need to check if 27 can be written as the sum of three consecutive integers. Let me check that. 

Suppose 3x + 3 = 27. Then 3x = 24 => x = 8. Therefore, the three numbers would be 8, 9, 10. Let's add them up: 8 + 9 + 10 = 27. Hey, that works! So 27 is indeed the sum of three consecutive integers. 

But wait, hold on. The question says \"the smallest positive perfect cube that can be written as the sum of three consecutive integers.\" If 27 works, then is that the answer? But before I jump to conclusions, maybe there's a smaller cube that also satisfies this? Let me check.

The next smaller cube after 27 is 8 (2³). Let's see if 8 can be written as the sum of three consecutive integers. 

Setting up 3x + 3 = 8. Then, 3x = 5 => x = 5/3, which is not an integer. Since x has to be an integer, this doesn't work. So 8 is out. 

The next is 1 (1³). Let's see, 3x + 3 = 1 => 3x = -2 => x = -2/3. Again, not an integer. So 1 is out. 

Therefore, the next cube is 64 (4³). Let's check that. Wait, but 64 is bigger than 27. Since 27 already works, but maybe there is another cube between 8 and 27? Since 8, 27, 64 are the cubes after 1, and 27 is the first one here. 

Wait, hold on. The cubes are 1, 8, 27, 64, 125, etc. So between 8 and 27 is 27 which works. So 27 is the first one after 8. But does 27 itself work? Yes, as we saw. 

But hold on, maybe the question is considering three consecutive integers in a different way? For example, maybe the integers can be negative? Let me check. 

Suppose we allow negative integers. Let's see if there's a smaller cube that could be expressed as the sum of three consecutive (including negative) integers.

Let's check 1. 1 can't, as above. 8 can't. 27 can. 64? Let's see. Let me check 64.

3x + 3 = 64. 3x = 61 => x = 20.333... Not integer. So no. 125? 3x + 3 = 125 => 3x = 122 => x = ~40.666 nope.

So even with negative numbers, the next cube after 27 might still be 27. But let me think again. Maybe there's a different way to write three consecutive integers? 

Wait, perhaps three consecutive numbers can be written as n-1, n, n+1? That would still be the same as x, x+1, x+2 where x = n-1. So yeah, same thing.

Alternatively, maybe three consecutive even or odd integers? But that would complicate the sum, but the problem didn't specify. The question just says three consecutive integers, so they can be any consecutive integers, whether positive or negative.

Alternatively, maybe I made a mistake in assuming that 27 is the first one. Let me confirm again.

The sum of three consecutive integers is 3x + 3 = 3(x + 1). So the cube must be divisible by 3.

Therefore, the cube must be a multiple of 3. The first cube divisible by 3 is 27 (3³). So yes, that must be the case. 

Therefore, unless there is a smaller cube that is divisible by 3, but 1 is not, 8 is not, so 27 is the first one.

Therefore, the answer should be 27. But hold on, let me check with another approach to be sure.

Suppose the three consecutive integers are n, n, n+1. Wait, no. Three consecutive integers are n, n+1, n+2. So their sum is 3n + 3. So we need 3(n + 1) to be a perfect cube.

Therefore, 3(n + 1) = k³, where k is an integer. Therefore, n + 1 = k³ / 3.

Since n must be an integer, k³ must be divisible by 3. Therefore, k must be divisible by 3, because 3 is prime.

Therefore, let k = 3m, where m is an integer. Then:

n + 1 = (27m³)/3 = 9m³. 

Therefore, n = 9m³ - 1. 

Therefore, the three integers would be 9m³ -1, 9m³, 9m³ +1. Their sum is 9m³ -1 + 9m³ + 9m³ +1 = 27m³, which is indeed a cube. 

Therefore, the cubes are all multiples of 27. Therefore, the smallest such cube is when m=1, which is 27*1³ = 27. 

Therefore, 27 is indeed the smallest perfect cube that can be written as the sum of three consecutive integers. 

Wait, but let me check m=0? If m=0, then n = -1, so the numbers would be -1, 0, 1. Their sum is 0, which is not a positive cube. So m=1 is the first positive case. 

Therefore, this seems solid. 

Therefore, the answer is 27. 

But hold on, let me think again. The problem says \"smallest positive perfect cube.\" Since 27 is positive and a perfect cube, so it's the first one. 

Alternatively, maybe there's a cube before 27 that can be expressed as the sum of three consecutive integers? Let me confirm again. The cubes are 1, 8, 27, 64, 125,...

Checking 1: 3x +3 =1: 3x= -2, nope.

8: 3x +3=8: 3x=5, nope.

27 works, so yeah, 27 is the first. Therefore, I think that's the answer. 

Wait, but maybe I made a mistake in the equation. Let me recast the problem.

....... (We omit some reasoning steps here.)

Therefore, 27 is indeed the first. 

Therefore, I think the answer is 27. 

Wait, but the problem says \"smallest positive perfect cube\". Since 27 is positive and perfect cube, and there are no smaller cubes that work, then yes, 27 is the answer. 

Wait, but just to make sure, is 0 considered? But 0 is not positive. So 27 is the first. 

Therefore, I think I can confidently say that the smallest positive perfect cube that can be written as the sum of three consecutive integers is 27. 

**Final Answer**

\boxed{27}

\end{tcolorbox}

\end{document}